\documentclass{article}
\usepackage{sty/arxiv}

\usepackage{graphicx}
\usepackage[utf8]{inputenc} % allow utf-8 input
\usepackage{url}
\usepackage[T1]{fontenc}    % use 8-bit T1 fonts
\usepackage{booktabs}
\usepackage[colorlinks]{hyperref}
\usepackage{bm}
\usepackage{multirow}
\usepackage{colortbl}
\usepackage[export]{adjustbox}
\usepackage{amsmath,amssymb,amsthm}
\usepackage{algorithm}
\usepackage[noEnd=true,indLines=true,commentColor=black]{sty/algpseudocodex}
\usepackage{stmaryrd}
\usepackage{hhline}
\usepackage[font=footnotesize,subrefformat=parens]{subcaption}
\usepackage{arydshln}
\usepackage{tikz}
\usepackage{siunitx}
\usepackage{cleveref}
\usepackage{sty/named}
\usepackage{sty/mystyle}

\title{Pathfinding with Lazy Successor Generation}
\author{Keisuke Okumura$^{1,2}$ \\
  $^1$University of Cambridge\\
  $^2$National Institute of Advanced Industrial Science and Technology (AIST)\\
  \texttt{ko393@cl.cam.ac.uk} \\
}
\date{}

\begin{document}
\maketitle
\begin{abstract}
  We study a pathfinding problem where only locations (i.e., vertices) are given, and edges are implicitly defined by an oracle answering the connectivity of two locations.
  Despite its simple structure, this problem becomes non-trivial with a massive number of locations, due to posing a huge branching factor for search algorithms.
  Limiting the number of successors, such as with nearest neighbors, can reduce search efforts but compromises completeness.
  Instead, we propose a novel \lacas algorithm, which does not generate successors all at once but gradually generates successors as the search progresses.
  This scheme is implemented with $k$-nearest neighbors search on a k-d tree.
  \lacas is a complete and anytime algorithm that eventually converges to the optima.
  Extensive evaluations demonstrate the efficacy of \lacas, e.g., solving complex pathfinding instances quickly, where conventional methods falter.
\end{abstract}

\section{Introduction}
\begin{definition}
  \label{def:problem}
  We consider a pathfinding problem within a 2D workspace $\W \subseteq [0, 1]^2$, encompassing a set of locations $V = \{v_1, v_2, \ldots, v_n\}$ with each $v_i \in \W$.
  Let $s \in V$ be the start location and $t \in V$ be the goal, distinct from $s$.
  An oracle function $\connect: V \times V \mapsto \{\true, \false\}$ determines whether an agent can traverse between any two points.
  A \emph{solution} is a sequence of locations $\pi \defeq (u_1, u_2, \ldots, u_m)$, where $u_k \in V$, satisfying $u_1 = s$, $u_m = t$, and $\connect(u_k, u_{k+1}) = \true$ for each $k \in \{1, \ldots, m-1\}$.
\end{definition}

\Cref{fig:top} presents an example instance along with its solutions.
The cost of a solution, $\cost(\pi)$, is given by $\sum_{k = 1}^{m-1} \dist(u_k, u_{k+1})$, where \dist represents the Euclidean distance.
A solution, $\pi^\ast$, is termed \emph{optimal} if no other solution $\pi$ exists with $\cost(\pi) < \cost(\pi^\ast)$.
An algorithm is \emph{complete} if it always yields a solution when solutions exist, or otherwise indicates their absence.
An algorithm is \emph{optimal} if it always finds an optimal solution.

{
  \begin{figure}[t!]
    \centering
    \includegraphics[width=0.5\linewidth]{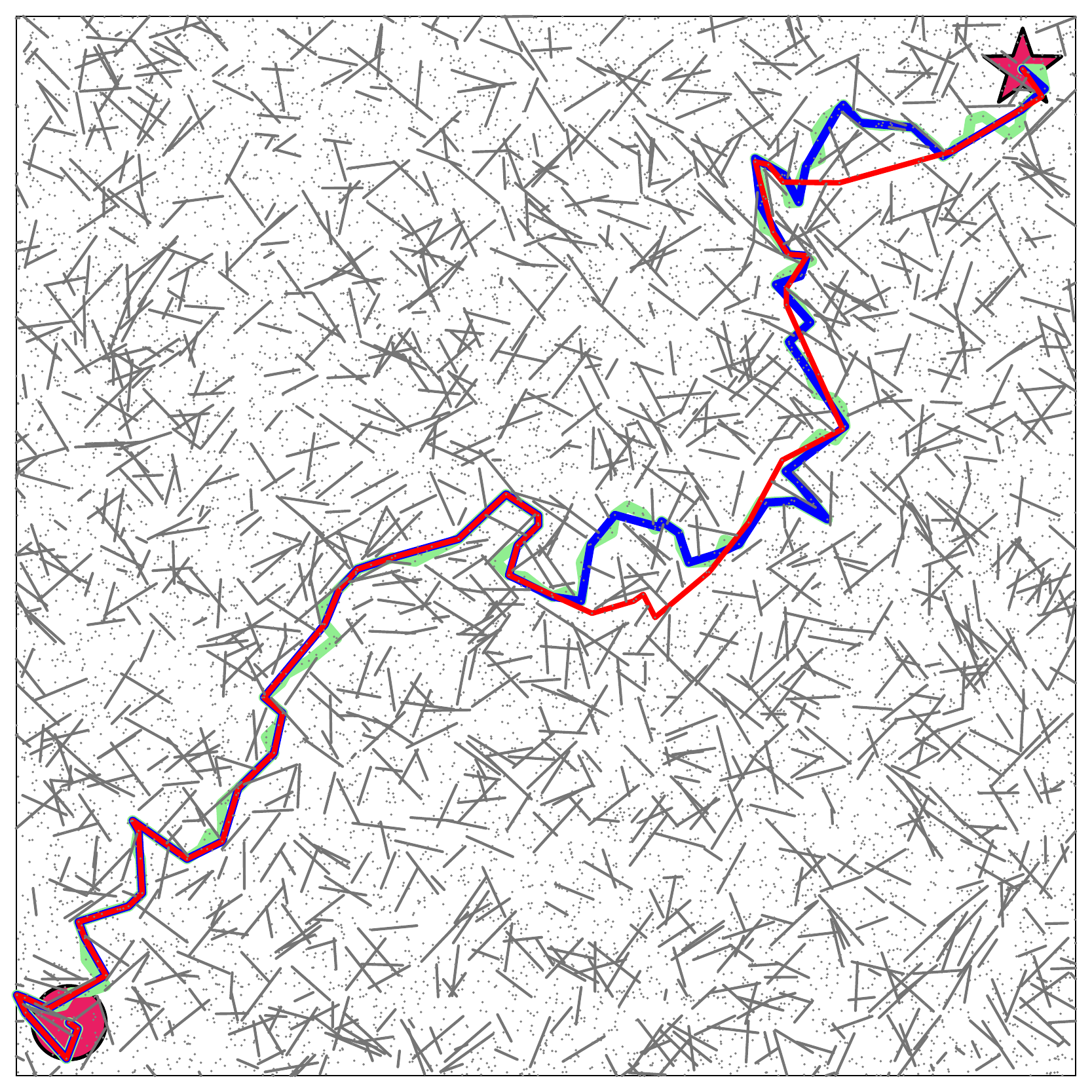}
    \caption{
      Example instance and solutions.
      The start is placed at the bottom left (circle), and the goal at the top right (star).
      Locations are denoted by small dots ($n=10,000$).
      The workspace includes obstacles represented by gray lines.
      Three solution paths are illustrated:
      \textcolor[HTML]{1AC938}{\emph{(i)} initial solution by \lacas},
      \textcolor[HTML]{023EFF}{\emph{(ii)} that of \lacat, a \lacas variant similar to \thetastar}~\protect\cite{daniel2010theta}, and
      \textcolor[HTML]{E8000B}{\emph{(iii)} refined \lacat path after \SI{10}{\minute}}.
      Both initial solutions were obtained in approximately \SI{30}{\second} with the Julia language implementation.
      For reference, \astar~\protect\cite{hart1968formal} did not find a path within \SI{1}{\hour}.
      Greedy best-first search required $\geq$\SI{9}{\minute} for a suboptimal path, and $\geq$\SI{3}{\minute} for RRT-connect~\protect\cite{kuffner2000rrt}.
    }
    \label{fig:top}
  \end{figure}
}

\Cref{def:problem} is also referred to as \emph{pathfinding on $E^4$-graphs (explicit graphs with expensive edge evaluation)}~\cite{choudhury2017densification}.
Solving \cref{def:problem} (near-)optimally and quickly is attractive, given its myriad applications in fields like robot motion planning and video games, among others.
However, \emph{this problem poses non-trivial challenges to search algorithms when $n$, the number of locations, is enormous, which is the interest of this paper}.
We first review existing related problems to underscore the unique features of \cref{def:problem}.

\subsection{Related Problems}

A \emph{graph pathfinding} problem involves determining the existence of a path in a graph $G = (V, E)$ that connects specified start and goal vertices, with $E$ denoting a set of arcs.
\Cref{def:problem} is framed as a graph pathfinding, where $V$ represents a set of locations, and $E$ is \emph{implicitly} defined by the \connect oracle.
Consequently, solving \cref{def:problem} optimally is achievable by classical search methods, such as \astar~\cite{hart1968formal}.
However, the primary challenge arises from each location's potential to connect to \emph{all} others, resulting in a huge \emph{branching factor} that hampers rapid solution derivation.
Moreover, frequent invocation of \connect can be computationally intensive, as seen in collision checking in motion planning studies~\cite{elbanhawi2014sampling}.
Limiting the number of neighbors for $v \in V$ might evade this issue, e.g. employing $k$-nearest neighbors of $v$ based on the Euclidean distance, but such methods sacrifice completeness and optimality.

\emph{Any-angle path planning on grids (AAPP)}~\cite{daniel2010theta} is a special case of \cref{def:problem} such that each location $v \in V$ is placed on a lattice grid.
AAPP has been extensively studied for its ability to generate shorter paths than standard grid pathfinding.
While there exists an efficient, optimal algorithm for AAPP~\cite{harabor2016optimal}, it is inapplicable to \cref{def:problem} due to its lattice grid assumption.
Adapting suboptimal approaches like \thetastar~\cite{daniel2010theta}, by restricting neighbor nodes, is feasible but again compromises completeness.

A class of \emph{lazy shortest path (LazySP)} algorithms~\cite{dellin2016unifying} aims to find the shortest path in graphs with the minimum edge evaluation overhead.
LazySP is motivated by the fact that edge evaluation is sometimes computationally expensive.
While it does not address the huge branching factor issue, solving \cref{def:problem} shares LazySP's motivation:
every location may connect to all others, leading to enormous edge evaluation times in naive brute-force approaches.

\emph{Motion planning (MP)}~\cite{lavalle2006planning} aims to find a collision-free path that leads an object to ideal states.
With a representation of configuration spaces, this problem is translated to ``single-point'' pathfinding in obstacle-free spaces.
\Cref{def:problem} is then understood as an MP problem, restricting agent configurations to predefined locations $V$ and transitions to those defined by the \connect oracle.
This resembles pathfinding on probabilistic roadmap (PRM)~\cite{kavraki1996probabilistic}, which approximates a configuration space by a graph representation comprising randomly sampled points, and \emph{explicitly} holds connections between locations.
However, standard PRM implementations often overlook the huge branching factor issue by assuming a connection between two points if they are sufficiently close.
An alternative MP method is tree-based planning, like rapidly exploring random tree (RRT)~\cite{lavalle1998rapidly}, which incrementally develops locations during the search process.
Yet, this class can struggle to locate suitable sampling points depending on the shapes of workspaces.

\subsection{Contribution}
The preceding discussions highlight a key dilemma:
On one hand, naive search algorithms that expand $O(n)$ successor nodes all at once from one location should be avoided due to substantial computational overhead.
On the other hand, examining connectivity across all potential locations is necessary to ensure the discovery of a solution path, if it exists.

A possible resolution to this dilemma is the gradual generation of successor nodes as the search progresses, termed \emph{lazy successor generation} in this paper.
Based on this notion, we introduce the \emph{\lacas} algorithm for solving \cref{def:problem}, an abbreviation for \emph{lazy constraints addition search in single-agent pathfinding}.
\lacas comprises a \emph{two-level search}:
the high-level search executes a search on the original problem, while the low-level search focuses on efficiently extracting a small portion of successors.
Specifically, the low-level utilizes $k$-nearest neighbors search on a \emph{k-d tree}~\cite{bentley1975multidimensional}, based on the premise that closer locations are more likely to be connected than distant ones.
Theoretically, \lacas is complete and is also an \emph{anytime} algorithm~\cite{zilberstein1996using} that initially finds unbounded suboptimal solutions and eventually converges to the optimum.
Empirical evaluations show that \lacas rapidly solves various complex problems with satisfactory quality, which are difficult to address with conventional pathfinding methods.

\paragraph{Paper Organization.}
\Cref{sec:algorithm} describes \lacas in details.
\Cref{sec:related-work} discusses related algorithms like \astar with partial expansion~\cite{yoshizumi2000partial,goldenberg2014enhanced} and \lacam for multi-agent pathfinding~\cite{okumura2023lacam,okumura2023lacam2}, which inspired this work.
\Cref{sec:tips} details techniques to implement \lacas, including \emph{\lacat} for enhanced solution quality.
\Cref{sec:evaluation} presents an empirical evaluation of \lacas, concluding in \cref{sec:discussions}.
The code will be available at the following website: \url{https://kei18.github.io/}.

\section{Algorithm}
\label{sec:algorithm}

\subsection{Lazy Constraints Addition}
\label{sec:lazy-constraints-addition}
The essence of lazy successor generation lies in incrementally generating small portions of successors, while ensuring all are eventually produced.
A naive approach involves initially sorting $O(n)$ locations by a distance relative to the target and goal, then sequentially extracting a \emph{batch}, a small portion of the entire locations, following the sort result.
However, this method demands substantial computation, as sorting each location requires an $O(n \log n)$ operation.
\lacas avoids this computationally intensive procedure by utilizing a k-d tree~\cite{bentley1975multidimensional}.

A k-d tree is a widely used data structure for storing points in k-dimensional space, efficiently answering queries to retrieve nearby points from a target point.
For instance, the time complexity of finding the nearest point among $n$ locations is $O(\log n)$ under certain reasonable assumptions, once the tree is constructed with an $O(n\log n)$ overhead.
Its applications span unsupervised classification tasks and sampling-based motion planning algorithms~\cite{lavalle2006planning}.
Here, it facilitates lazy successor generation.
Specifically, consider a k-d tree \T constructed for a set of locations $V$.
Then, \lacas employs the operation $\funcname{nearest\_neighbors\_search}(\T, v \in V, b \in \mathbb{N}_{>0}, \theta \in \mathbb{R})$ to retrieve $b$ nearest neighbors in $V$ of $v$, based on distance, but each with a distance exceeding $\theta$.
If $\theta < 0$, it equates to the standard $k$-nearest neighbors search for \T.
Implementing \funcname{nearest\_neighbors\_search} is straightforward, because it merely introduces a threshold distance $\theta$ to the normal $k$-nearest neighbors search.
We assume that the function returns an empty set if no locations meet the threshold.

Observe that \funcname{nearest\_neighbors\_search} functions as a lazy successor generator.
It circumvents the need to enumerate all locations at once, allowing for the eventual generation of all locations through multiple invocations by incrementally increasing the threshold distance $\theta$.
In this context, $\theta$ acts as a \emph{constraint} on successor generation, with the constraints being added in a \emph{lazy} manner.
The name of \lacas, \emph{lazy constraints addition search}, comes from this view.

\subsection{Pseudocode}
\Cref{algo:lacas} shows a minimum \lacas.
The anytime parts are grayed out and explained after providing the backbone.

Similar to other search algorithms, a search node in \lacas consists of a location $v \in V$ and a pointer to its parent node; see \cref{algo:lacas:init-node}.
Search nodes are stored in an \open list, which directs the search, and an \explored table to prevent duplicating nodes.
\open can be implemented using various data structures such as stacks, queues, etc.
The following description assumes a stack, adhering to a depth-first search (DFS) style.

After constructing the k-d tree (\cref{algo:lacas:build-kdtree}), \lacas progresses the search by processing one node \N at a time (\cref{algo:lacas:top}).
The procedure initially retrieves $b$ potential successor locations for \N (\cref{algo:lacas:successors}).
If no such locations exist, the original node \N is discarded (\cref{algo:lacas:discard});
otherwise, the distance threshold is updated based on these locations (\cref{algo:lacas:update-threshold}).
Subsequently, successor nodes are generated after confirming connectivity (Lines~\ref{algo:lacas:for-new-node}--\ref{algo:lacas:insert-new-node}).
When the search reaches a goal location $t$ (\cref{algo:lacas:hit-target}), a solution can always be constructed by backtracking from the node at the goal ($\N\goal$; \cref{algo:lacas:opt,algo:lacas:subopt}).
If nodes are exhausted without reaching $t$, it signifies the absence of a solution (\cref{algo:lacas:no-solution}).

{
  \newcommand{\dopen}{\m{\mathcal{Q}}}
  \begin{algorithm}[t!]
    \caption{\lacas}
    \label{algo:lacas}
    \begin{algorithmic}[1]
      \Input{$V$, $s$, $t$, \connect}
      \Output{path, \nosolution, or \failure~{\small (e.g., timeout)}}
      \Params{batch size $b \in \mathbb{N}_{>0}$}
      \Notation{\gl{$\f(\N) \defeq \N.g + \dist(\N.v, t)$;} \gl{$\spadesuit \defeq (\N\goal$$\neq$$\bot)$}}
      \State initialize \open, \explored;~\gl{$\N\goal \leftarrow \bot$}
      \State
      $\N\init \leftarrow \begin{aligned}[t]
        \bigl\langle
        v: s,
        \parent: \bot,
        \theta: 0,
        \gl{g: 0,}
        \gl{\neighbors: \emptyset}
        \bigr\rangle\end{aligned}$
      \label{algo:lacas:init-node}
      \State $\open.\push(\N\init)$;~~$\explored[s] = \N\init$
      \label{algo:lacas:init-node-insert}
      \State $\T \leftarrow \funcname{build\_kdtree}(V)$
      \label{algo:lacas:build-kdtree}
      \While{$\open \neq \emptyset~\land~\lnot\funcname{interrupt}()$}
      \State $\N \leftarrow \open.\funcname{top}()$
      \label{algo:lacas:top}
      \IfSingle{$\N.v = t$}{$\N\goal \leftarrow \N$}
      \label{algo:lacas:hit-target}
      \IfSingleGl{$\spadesuit \land \f(\N) \geq \f(\N\goal)$}{$\open.\pop()$; \Continue}
      \label{algo:lacas:pruning}
      %
      % lazy successor generation
      \State $\B \leftarrow \funcname{nearest\_neighbors\_search}(\T, \N.v, b, \N.\theta)$
      \label{algo:lacas:successors}
      \IfSingle{$\B = \emptyset$}{$\open.\pop()$;~\Continue}
      \label{algo:lacas:discard}
      \State $\N.\theta \leftarrow \max_{v \in \B}\dist(\N.v, v)$
      \label{algo:lacas:update-threshold}
      %
      % neighborhood
      \For{$v \in \B$}
      \label{algo:lacas:for-new-node}
      \IfSingle{$\lnot\connect(\N.v, v)$}{\Continue}
      \If{$\explored[v] = \bot$}
      \StateGl{$g \leftarrow \N.g + \dist(\N.v, v)$}
      \State $\N\new$$\leftarrow$$\langle
        v$$: v,
        \parent$$: \N,
        \theta$$: 0,
        \gl{g: g,}
        \gl{\neighbors: \emptyset}
        \rangle$
        \label{algo:lacas:create-new-node}
      \State $\open.\push(\N\new)$;\;$\explored[v] = \N\new$
      \label{algo:lacas:insert-new-node}
      \StateGl{$\N.\neighbors.\append(\N\new)$}
      \Else
      \StateGl{$\N.\neighbors.\append(\explored[v])$}
      \label{algo:lacas:update-existing-node-neighbor}
      \StateGl{$\dopen \leftarrow \llbracket \N \rrbracket$}
      \Comment{Dijkstra update, for optimality}
      \label{algo:lacas:dijkstra}
      \WhileGl{$\dopen \neq \emptyset$}
      \StateGl{$\N\from \leftarrow \dopen.\pop()$}
      \ForGl{$\N\to \in \N\from.\neighbors$}
      \StateGl{$g \leftarrow \N\from.g + \dist(\N\from.v, \N\to.v)$}
      \IfGl{$g < \N\to.g$}
      \label{algo:lacas:g-value-pruning}
      \StateGl{$\N\to.g \leftarrow g$;\;$\N\to.\parent \leftarrow \N\from$}
      \StateGl{$\dopen.\push(\N\to)$}
      \IfGl{$\spadesuit \land \f(\N\to) < \f(\N\goal)$}
      \label{algo:lacas:revive-start}
      \StateGl{$\open.\push\left(\N\to\right)$}
      \EndIfGl
      \label{algo:lacas:revive-end}
      \EndIfGl
      \EndForGl
      \EndWhileGl
      \label{algo:lacas:dijkstra-end}
      \EndIf
      \EndFor
      \EndWhile
      \If{$\spadesuit \land \open = \emptyset$}
      \Return $\backtrack(\N\goal)$
      \Comment{opt.}
      \label{algo:lacas:opt}
      \ElsIf{$\spadesuit$}
      \Return $\backtrack(\N\goal)$
      \Comment{suboptimal}
      \label{algo:lacas:subopt}
      \ElsIf{$\open = \emptyset$}
      \Return \nosolution
      \label{algo:lacas:no-solution}
      \Else
      ~\Return \failure
      \Comment{e.g., timeout}
      \EndIf

    \end{algorithmic}
  \end{algorithm}
}

\paragraph{Anytime Components.}
To eventually find optimal solutions, each search node includes a $g$-value denoting the cost-to-come, and \neighbors, holding discovered successors.
Upon encountering previously known locations, \lacas rewires parent pointers using an adapted Dijkstra algorithm~\cite{dijkstra1959note} (Lines~\ref{algo:lacas:dijkstra}--\ref{algo:lacas:dijkstra-end}).
This rewiring is typically applied to a small section of the search tree due to pruning by $g$-values (\cref{algo:lacas:g-value-pruning}).
Another anytime component is that upon reaching the goal, \lacas excludes nodes not improving solution quality to accelerate solution refinement (\cref{algo:lacas:pruning}).
This exclusion relies on the $f$-value, i.e., the aggregate of the $g$-value and the estimated cost-to-go (i.e., $h$-value; distance to the goal; aka. heuristic).
It also restores nodes when needed, following updates to their $g$-values (\cref{algo:lacas:revive-start,algo:lacas:revive-end}).

\paragraph{Remarks.}
\Cref{algo:lacas} assumes that \emph{each location has unique distances to other locations};
otherwise, lazy successor generation may overlook the generation of specific successors when two locations are equidistant, as in grid worlds.
This issue can be easily addressed with deterministic tiebreaking for distances.
Indeed, our implementation uses such a tiebreaking method.
Thus, subsequent analysis implicitly employs this assumption of unique distances.

\subsection{Analysis}
\begin{theorem}
  \lacas (\cref{algo:lacas}) is complete and optimal.
  \label{thrm:lacas}
\end{theorem}
\begin{proof}
  \emph{Completeness:}
  Each time a node \N is invoked, its distance threshold $\theta$ monotonically increases by updating $\theta$ to the maximum distance from $\N.v$ to potential successor locations in $\B$ (\cref{algo:lacas:update-threshold}), each having at least the distance of $\theta$.
  Since $\theta$ initially starts at zero, any location $u \in V$, with $u \neq \N.v$, must have been encompassed in one of \N's invocations by the time \N is discarded at \cref{algo:lacas:discard}.
  Furthermore, each search node must be discarded after a finite number of invocations as $n$ is finite.
  Hence, \lacas performs an exhaustive search within a finite space, ensuring completeness.

  \emph{Optimality:}
  Consider \cref{algo:lacas} excluding the pruning components (\cref{algo:lacas:pruning,algo:lacas:revive-start,algo:lacas:revive-end}), because they merely accelerate the search without compromising the complete and optimal search structure.
  For simplicity, let's assume $b=1$.
  The general case is straightforward.

  The proof employs a directed graph $G$, each vertex of which represents a location.
  Initially, $G$ comprises solely the start location $s$.
  Subsequently, the search iterations gradually develop $G$.
  Upon the generation of a new node (\cref{algo:lacas:create-new-node}), its corresponding location is incorporated into $G$.
  An arc $(u, v)$ in $G$ happens when the search identifies a connection from $u$ to $v$, i.e., $\N^v \in \N^u.\neighbors$.

  We now prove that: \emph{($\clubsuit$) for any node $v$ in $G$, a path from $s$ to $v$ traced by backtracking (i.e., following $\N.\parent$) represents the shortest path in $G$, considering cumulative edge costs.}
  This is proven by induction.
  Initially, $G$ consists solely of $s$, thus fulfilling $\clubsuit$.
  Assume now that $\clubsuit$ is satisfied in the previous search iteration.
  In the next iteration for the search node \N, $G$ is updated by either:
  \emph{(i)}~finding a new location, or
  \emph{(ii)}~finding a known location.
  Case \emph{(i)} holds $\clubsuit$ because only a new location and an arc toward the location are added to $G$.
  Case \emph{(ii)} holds $\clubsuit$ because the addition involves solely a new arc from \N to the known location, followed by the adaptive implementation of Dijkstra's algorithm from \N, preserving the shortest path tree structure.

  Upon completion of the search, $G$ entails all feasible paths from $s$ to $t$, as \lacas performs an exhaustive search.
  In conjunction with $\clubsuit$, \lacas yields an optimal solution for solvable instances, otherwise, reports the non-existence.
\end{proof}

The preceding analysis regards \lacas as an optimal algorithm; however, in practice, a suboptimal solution can be obtained at any point once after the search hits the goal.
As time allows, \lacas continues to improve the solution quality.
This feature is particularly useful for real-time systems where planning time is severely limited.

\section{Related Work}
\label{sec:related-work}
This section organizes relationships with existing algorithms before going to the empirical side.

\subsection{Beam Search}
We begin the literature review from \emph{beam search}~\cite{bisiani1987beam}, which is a heuristic search method that limits the number of nodes in the \open list.
It can be categorized into two types~\cite{wilt2010comparison}:
\emph{(i)} the best-first type, which merely limits the size of the \open compared to the standard best-first search, and
\emph{(ii)} the breadth-first type, which works like breadth-first search but with a predetermined width, known as ``beam width.''
Although beam search sacrifices completeness and optimality, the breadth-first variant has been widely used in areas such as speech recognition~\cite{ravishankar1996efficient} and natural language processing~\cite{cohen2019empirical} because of its simplicity, efficiency in reducing search effort, and memory consumption.
Researchers have been extended basic beam search to have theoretical properties such as completeness or optimality~\cite{zhang1998complete,zhou2005beam,furcy2005limited,vadlamudi2013incremental}.

Similar to beam search, \lacas does not retain all successor nodes;
however, there are two significant differences.
First, unlike beam search, which generates all successors before pruning, \lacas initiates only a subset of successors through a two-level search approach.
Second, while beam search parameters (e.g., beam width) determine the number of nodes maintained throughout the search, \lacas's batch size specifies the number of successors for each search node.
These features allow \lacas to facilitate pathfinding with numerous successors while maintaining theoretical properties, i.e., to adapt the search effort to the area of interest without spending too much time generating successors.

\subsection{Partial Successor Expansion}
Rather than beam search, \lacas has more close relationships to the best-first search with \emph{partial successor expansion} as follows.

Partial Expansion \astar (\pea)~\cite{yoshizumi2000partial} is an \astar variant.
To keep memory usage small, when expanding a search node \N, \pea generates \emph{all} successors and inserts only a part of them, those having the same $f$-value as $\N$, into \open.
In our context, this approach closely aligns with the strategy outlined at the beginning of \cref{sec:lazy-constraints-addition}, which initially orders all locations relative to \N, and subsequently selects segments of this ordered list as the search progresses.
Nonetheless, sorting can become a significant bottleneck when $n$ is huge.
\lacas circumvents this bottleneck by efficiently leveraging search with the k-d tree.

Enhanced \pea (\epea)~\cite{goldenberg2014enhanced} augments \pea by considering the time aspect.
Specifically, \epea produces \emph{only} successors that have the same $f$-value as their parent node.
Such successors are identified using domain-specific knowledge, like in a sliding puzzle, making \epea not universally applicable to planning problems such as \cref{def:problem}.
\lacas principally inherits the concept of lazy successor generation from \epea.
However, a significant departure from \epea is that \lacas no longer cares about $f$-values when generating successors.
Instead, it employs $f$-values for pruning after the discovery of initial solutions.
This approach simplifies the implementation of lazy successor generation in \lacas compared to \epea.

\lacam~\cite{okumura2023lacam,okumura2023lacam2} is a scalable algorithm for multi-agent pathfinding (MAPF)~\cite{stern2019def}, which aims to find collision-free paths for multiple agents on graphs.
By considering a joint search space for all agents, MAPF is understood as a graph pathfinding problem with a huge branching factor.
Following the remarkable success of \lacam, albeit tailored for MAPF, we posited that the underlying principles of \lacam could be applied to \cref{def:problem}, since those two problems have the common issue of the huge branching factor.
The resultant algorithm is \lacas.
The two algorithms vary in their implementations on lazy constraints addition and successor generation, designed for their respective problem domains.
Meanwhile, they are both complete algorithms for their domains, and eventually converge to optimal solutions starting from unbounded suboptimal ones.

\subsection{Search Space Densificaiton}
Another direction for tackling search problems with a huge branching factor is to tailor the search spaces themselves before applying search algorithms;
that is, to first present a sparse search space and gradually \emph{densify} it.
This approach is orthogonal to \lacas, since they are not search algorithms;
but we include this discussion for the sake of completeness.

A concrete example is presented in~\cite{choudhury2017densification}, which explores pathfinding in $E^4$-graphs, a generalization of \cref{def:problem}.
The essence of this study is to construct a set of subgraphs from the original graph and solve pathfinding on these subgraphs sequentially.
A subgraph is densified from the previous one by decreasing the minimum distance required to connect two locations or by increasing the number of vertices included.
Similar to \lacas, this strategy employs an anytime planning scheme that aims to first find feasible solutions in a sparse representation of the workspace and then gradually refine the solution using denser representations.
Meanwhile, this workspace densification strategy is agnostic to the search algorithm;
it densifies the search space uniformly.
In contrast, \lacas is a search algorithm that automatically densifies the specific space of interest.

In line with this approach, another notable example is introduced in~\cite{saund2020fast}, where the workspace is represented by layered graphs, providing a finer representation of the space in deeper layers.
Each layer is connected to the next, and together these graphs define the huge search space.
Although agnostic to search algorithms, the layered representation facilitates the natural densification of regions of interest.
However, the construction of layers assumes specific location patterns, such as grids or Halton sequences, which do not apply to \cref{def:problem}.
Moreover, this approach does not directly address the problem of the huge branching factor, since it does not involve edges between any two arbitrary locations.

\subsection{Search Tree Rewiring}
\lacas incorporates tree rewiring in order to eventually obtain optimal solutions.
This scheme modifies overestimated $g$-values for each search node, removing inconsistencies caused by `shortcut' paths identified as the search progresses.
This notion essentially derives from lifelong planning~\cite{koenig2004lifelong,koenig2005fast}, where an agent must continuously find solution paths in dynamically changing environments.

The rewiring in \lacas resembles asymptotically optimal sampling-based motion planning (SBMP) algorithms like \rrtstar~\cite{karaman2011sampling} and its variants in discretized search spaces~\cite{shome2020drrt}.
While these algorithms limit rewiring to local neighbors, \lacas extends it to descendant nodes.
This difference arises because SBMP presumes an infinite number of samples (i.e., locations), whereas \lacas operates under a finite number of search iterations.

Notably, \lacas shares another feature with SBMP: node pruning based on the known solution cost (\cref{algo:lacas:pruning}).
For instance, informed SBMP methods such as Informed \rrtstar~\cite{gammell2014informed} bias new location sampling from a region identified by the known solution cost.
Such pruning is often seen in anytime search algorithms as well~\cite{hansen2007anytime,van2011anytime}.

{
  \newcommand{\entry}[2]{
    \begin{minipage}{0.21\linewidth}
      \centering
      \includegraphics[trim=11 11 11 11, clip,width=1.0\linewidth]{fig/raw/#1}
      {#2}
      \medskip
    \end{minipage}
  }
  \newcommand{\w}[1]{\textbf{#1}}
  \begin{figure}[t!]
    \centering
    {
      \small
    \setlength{\tabcolsep}{0pt}
    \begin{tabular}{ccccc}
      \multicolumn{2}{c}{\S\ref{subsec:order}}
      &&
      \multicolumn{2}{c}{\S\ref{subsec:reinsert}}
      \\
      \entry{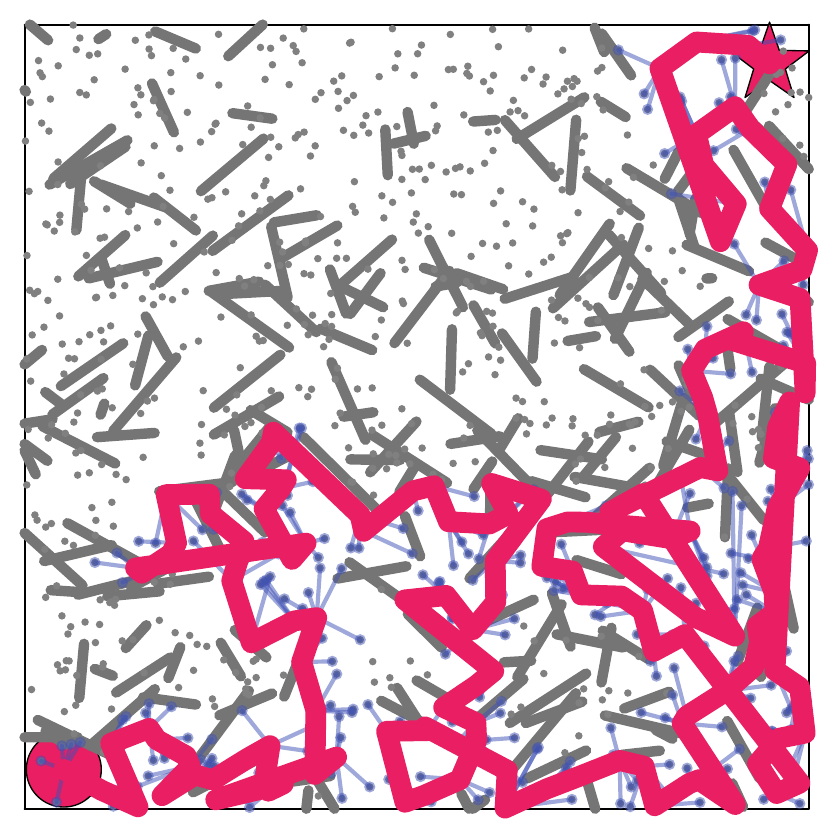}{random} &
      \entry{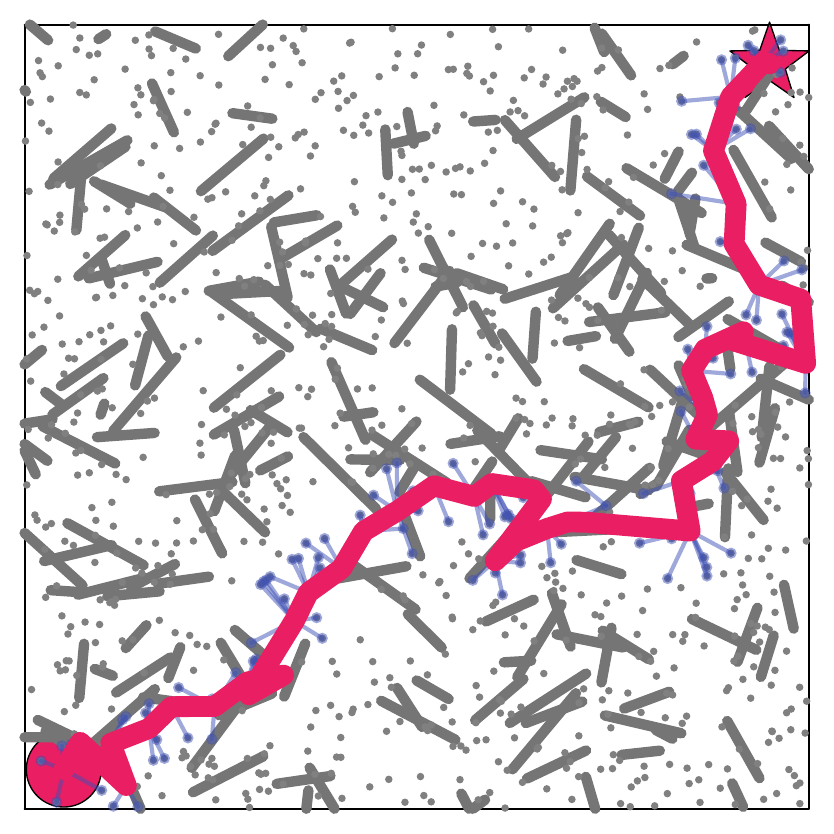}{sorted} &\;\;\;\;\;\;\;\;&
      \entry{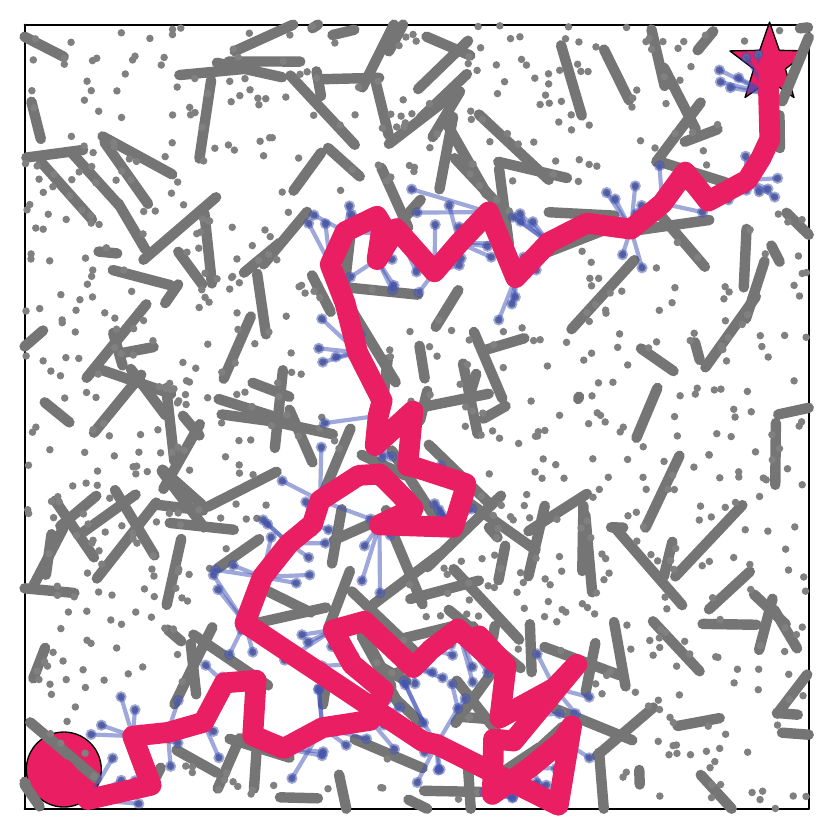}{sorted} &
      \entry{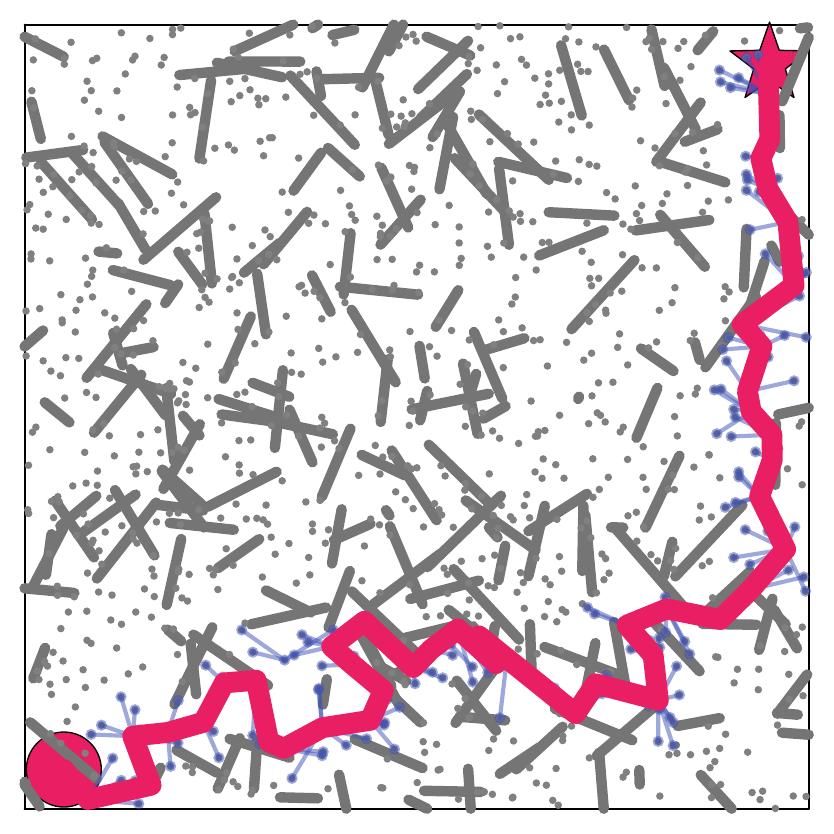}{reinsert}
      \\
      \multicolumn{2}{c}{\S\ref{subsec:rolling}}
      &&
      \multicolumn{2}{c}{\S\ref{subsec:theta}}
      \\
      \entry{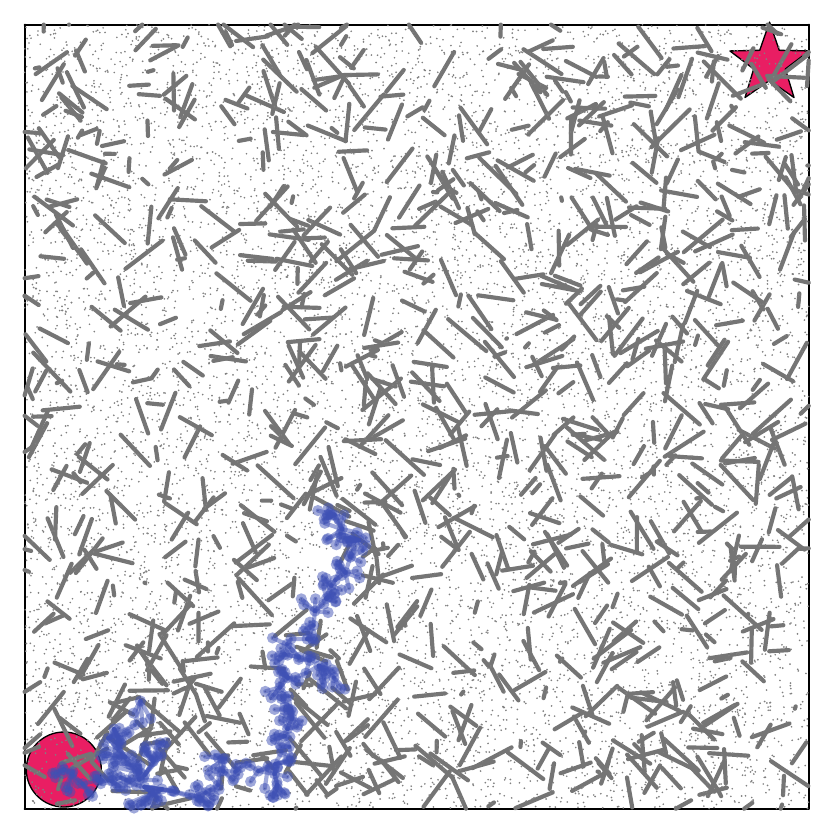}{reinsert} &
      \entry{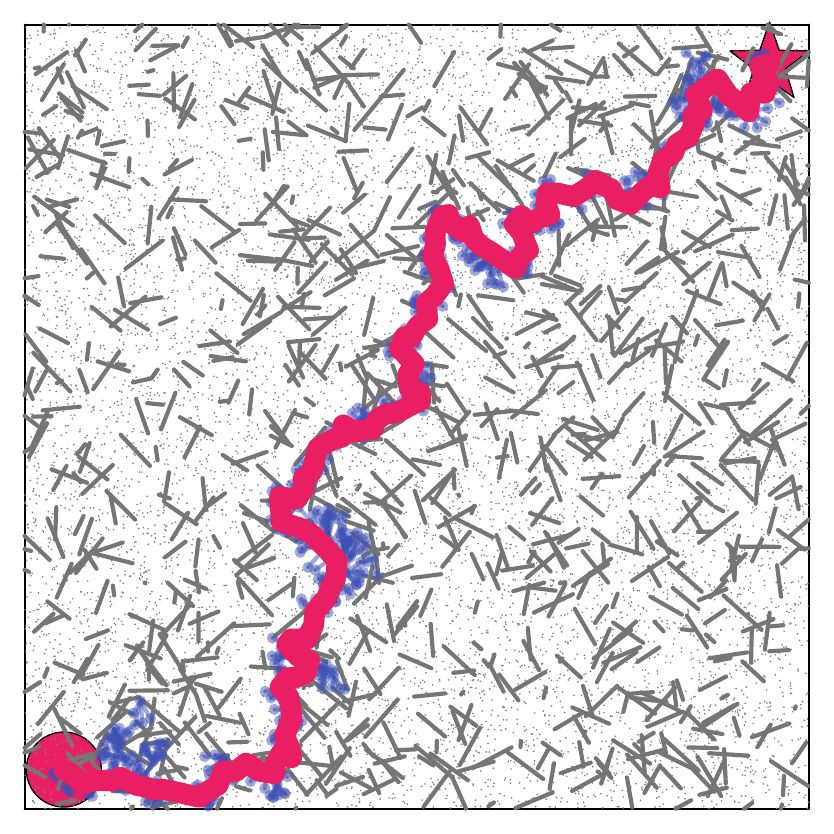}{rolling} &\;\;\;\;\;\;\;\;&
      \entry{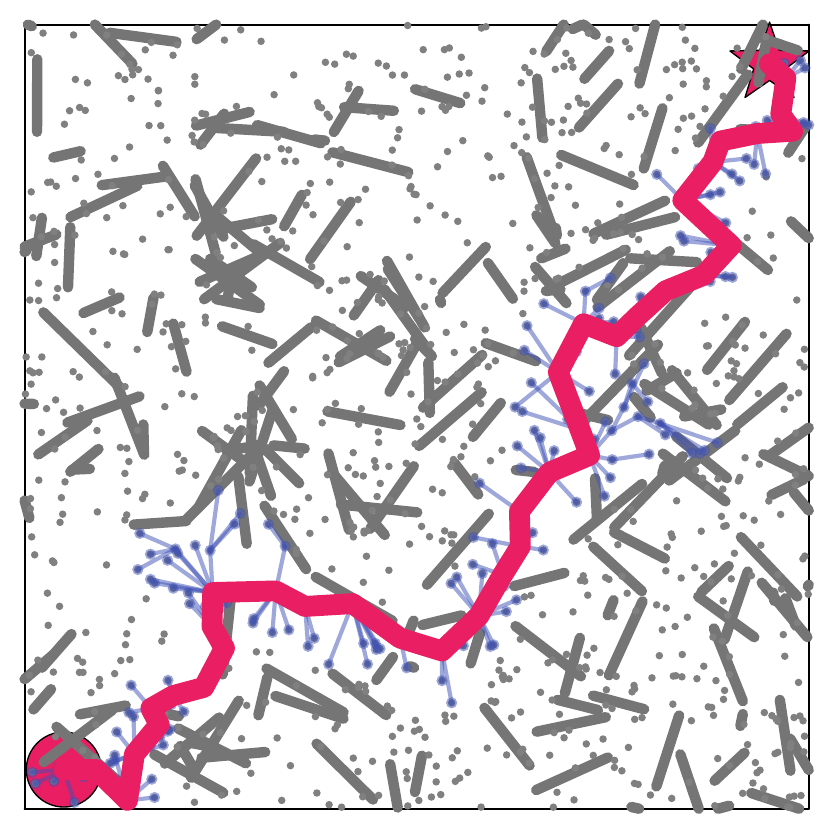}{rolling} &
      \entry{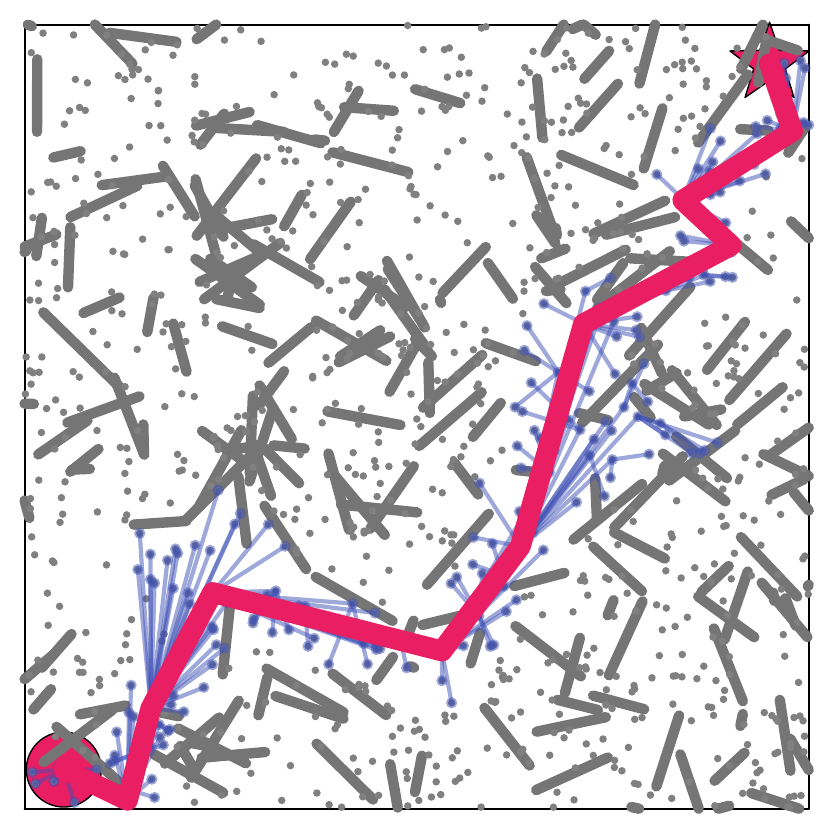}{\lacat}
      \vspace{0.5cm}
    \end{tabular}
    }
    {
    \begin{tabular}{rrrrrrrrr}
    \toprule
    &
    & solved{(\%)}$\uparrow$
    & time{(\SI{}{\second})}$\downarrow$
    & cost$\downarrow$
    & \#\connect$\downarrow$
    & \#iteration$\downarrow$
    \\\midrule
    \multirow{2}{*}{\S\ref{subsec:order}}
    & random & 73 & 8.00 & 10.94 & 52633 & 5316
    \\
    & sorted & \w{79} & \w{1.11} & \w{4.25} & \w{6695} & \w{677}
    \\
    \midrule
    \multirow{2}{*}{\S\ref{subsec:reinsert}}
    & sorted & 79 & 1.08 & 4.17 & 6470 & 654
    \\
    & reinsert & \w{83} & \w{0.35} & \w{1.90} & \w{1165} & \w{118}
    \\
    \midrule
    \multirow{2}{*}{\S\ref{subsec:rolling}}
    & reinsert & 66 & 11.66 & 1.97 & 11738 & 1176
    \\
    & rolling & \w{84} & \w{7.68} & 1.97 & \w{4603} & \w{461}
    \\
    \midrule
    \multirow{2}{*}{\S\ref{subsec:theta}}
    & rolling & 82 & \w{0.34} & 1.88 & \w{948} & 96
    \\
    & \lacat & 82 & 0.56 & \w{1.68} & 1481 & 96
    \\
    \bottomrule
    \end{tabular}
    }
    \medskip
    \caption{
      Effect of techniques in \cref{sec:tips}.
      \emph{upper:}
      \textcolor[HTML]{023EFF}{Search progress} is illustrated with dots and fine lines, representing locations within \explored when finding initial solutions or upon reaching the \SI{30}{\second} time limit.
      \textcolor[HTML]{E91E63}{Solutions} are marked by bold lines.
      These experiments used the \scen{scatter-1k} scenario, except for \S\ref{subsec:rolling} which used \scen{scatter-10k}, both detailed in \cref{sec:evaluation}.
    \emph{lower:}
    Quantitative results are summarized.
    Each scenario comprised 100 instances, potentially including unsolvable ones.
    ``solved'' indicates the number of instances solved within the time limit.
    Other metrics average the results of instances both methods successfully solved.
    ``time'' denotes the runtime to find initial solutions.
    ``cost'' refers to the initial solution costs.
    ``\#\connect'' is the number of times the method used \connect to derive initial solutions, and ``\#iteration'' is that of the number of search iterations.
    }
    \label{fig:devising}
  \end{figure}
}

\section{Devising Implementation}
\label{sec:tips}
\Cref{algo:lacas} is a minimal \lacas that leaves room for several inventions for implementation, which are covered in this section.
The techniques below rely on the fact that \lacas makes no assumptions about the node extraction order from \open for both guarantees of completeness and optimality.

\subsection{Order within Batch}
\label{subsec:order}

In each iteration, \lacas forms a batch \B at \cref{algo:lacas:successors} then processes locations in \B sequentially.
The order of locations in \B is flexible; however, it significantly impacts performance.
Among various design choices, we particularly choose to sort locations in \B in descending order of $\dist(v, t)$, where $v \in \B$, to derive initial solutions rapidly.
Using a stack structure in \open, the subsequent search iteration prioritizes the location in \B nearest to the goal, akin to greedy best-first search.

Using the experimental setup detailed in \cref{sec:evaluation}, we empirically evaluated two enumeration methods for \B: \emph{(i)}~random order and \emph{(ii)}~order sorted by distance to the goal.
\Cref{fig:devising} clearly illustrates the impact of these orders, both qualitatively and quantitatively, on initial solution discovery.

\subsection{Node Reinsert}
\label{subsec:reinsert}
When the search encounters an already known location, reinserting its corresponding node at the top of \open can enhance search efficiency.
This method, adopted from LaCAM~\cite{okumura2023lacam}, is based on the premise that frequently encountered locations during the search are potential bottlenecks.
Thus, prioritizing the search at these locations by reinvoking the corresponding node is logical.
We assume that this reinsertion occurs immediately after rewiring (Lines~\ref{algo:lacas:dijkstra}--\ref{algo:lacas:revive-end}).
\Cref{fig:devising} presents empirical results, showing that the node-reinsert operation decreases search effort and enables \lacas to discover superior initial solutions.

\subsection{Node Rolling}
\label{subsec:rolling}
Using a stack structure as \open may lead to situations where the same location is visited repeatedly in a short period because \lacas does not immediately discard invoked nodes.
This hinders rapid pathfinding;
considering setting aside such stuck nodes temporarily might be beneficial.
An implementation of this concept is adopting a double-ended queue (deque) for \open instead of the naive stack, allowing for \emph{rolling} search nodes within \open upon invocation.
Specifically, this operation pops the search node $\N$ from \open and reinserts it at the bottom if $\N$ generates a non-empty batch at \cref{algo:lacas:successors}, assuming to happen just after \cref{algo:lacas:discard}.
\Cref{fig:devising} proves that node rolling is effective for tackling complex problems.

\subsection{\lacat -- Check Grandparent}
\label{subsec:theta}
Consider a new node $\N\new$ created at \cref{algo:lacas:create-new-node} from a node $\N$.
If $\N\new$ can connect from $\N.\parent$, setting $\N.\parent$ as its parent yields a lower $g$-value for $\N\new$ compared to using $\N$ as the parent, because of the cost function adhering to the triangle inequality.
This grandparent check accelerates the attainment of better solutions compared to the vanilla \lacas, while maintaining theoretical integrity of \lacas as long as adding $\N\new$ to both $\N.\neighbors$ and $\N.\parent.\neighbors$.
This devising is also applicable when encountering existing nodes at \cref{algo:lacas:update-existing-node-neighbor}.
The grandparent check is inspired by \thetastar for AAPP~\cite{daniel2010theta}.
Taking its initial, we call the resulting algorithm \emph{\lacat}.

\Cref{fig:devising} demonstrates that \lacat can obtain superior initial solutions in the same number of iterations, albeit with the overhead of more \connect calls.
It is important to note that the preference for \lacat varies depending on the case;
rather than allocating time to the overhead of \lacat, it may be more beneficial to dedicate that time to the refinement scheme of \lacas.

\section{Evaluation}
\label{sec:evaluation}
This section assesses \textbf{\lacas} and its variant \textbf{\lacat}, both with techniques from \cref{sec:tips}, across various scenarios.
For clarity, \textbf{LaCAS} refers to \lacas without refinement after finding initial solutions, and \textbf{LaCAT} to that of \lacat.
Unless mentioned, LaCAS-based methods used a batch size $b$ of $10$.
Experiments were conducted on a 28-core desktop PC with an Apple M1 Ultra \SI{2.4}{\giga\hertz} CPU and \SI{64}{\giga\byte} RAM, with a \SI{30}{\second} timeout.
All methods were coded in Julia.

\subsection{Baselines}
We carefully selected various baseline methods, each belonging to one of two categories.
The first category includes methods that directly solve \cref{def:problem}, into which \lacas falls.
\begin{itemize}
\item \textbf{\astar}, a \emph{complete} and \emph{optimal} search scheme, using $\dist(\cdot, t)$ as a heuristic.
  We also evaluated \textbf{\astark}, limiting successors to the closest $k$ neighbors, and \textbf{\astarr}, limiting successors within distance $r$.
  Both are \emph{incomplete} and \emph{suboptimal}.
  Successors were efficiently identified using the k-d tree, similar to \lacas.
  For the experiments, we set $k=10$ and $r=0.1$, pre-adjusted to yield consistently superior performance.
\item Greedy best-first search (\textbf{GBFS}), which selects nodes solely based on the heuristic.
  This is \emph{complete} but \emph{suboptimal}.
  \textbf{GBFS-k} and \textbf{GBFS-r} were also tested, analogous to \astark and \astarr, respectively.
\item Depth-first search (\textbf{DFS}), a \emph{complete} and \emph{suboptimal} algorithm, conducted in a recursive fashion.
  The implementation prioritized visiting successors based on their distance to the goal to enhance performance.
\item \textbf{dRRT}~\cite{solovey2016finding}, an RRT variant tailored for discretized search spaces.
  It is \emph{probabilistically complete}, meaning the likelihood of finding a solution increases over time.
\item \textbf{PE}, a LaCAS variant, which first examines all successors and their heuristics, followed by the LaCAS scheme.
  For each node invoked, PE inserts a batch of successors into \open based on the distance threshold of that node.
  PE can be seen as LaCAS without the k-d tree search, akin to \pea~\cite{yoshizumi2000partial}.
\end{itemize}

{
  \setlength{\tabcolsep}{0.1pt}
  \newcommand{\entry}[3]{
    \begin{minipage}{0.12\linewidth}
      \centering
          {\small \textbf{\textcolor[HTML]{#3}{#2}}}\\
      \includegraphics[trim=11 11 11 11, clip,width=1.0\linewidth]{fig/raw/tree/#1}
    \end{minipage}
  }
  \begin{figure*}[t!]
    \centering
    \begin{tabular}{cccccccc}
      \entry{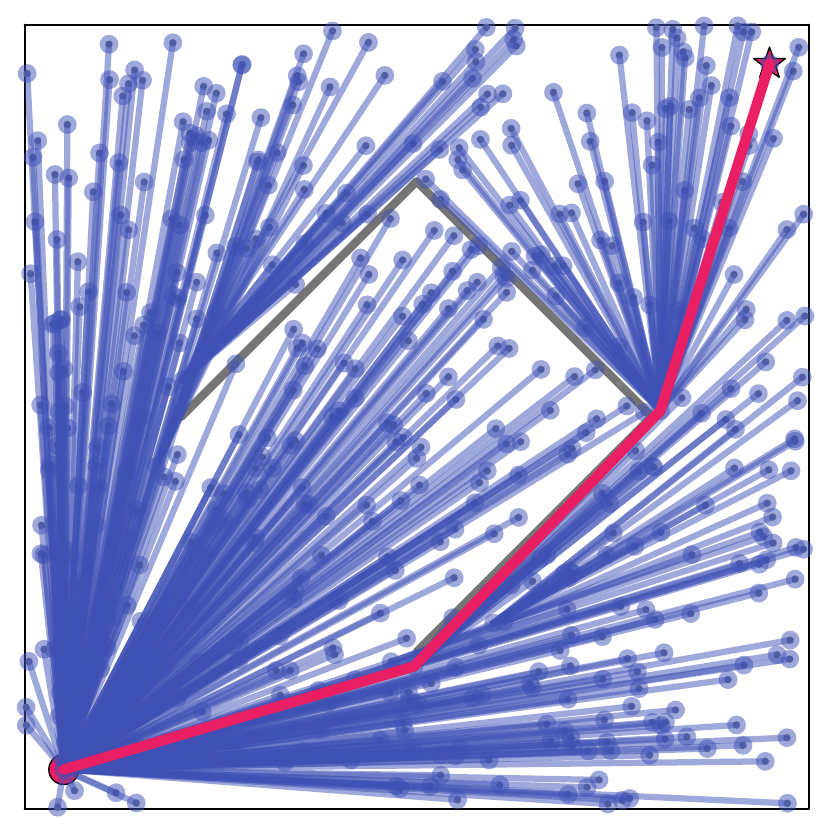}{\astar}{023EFF} &
      \entry{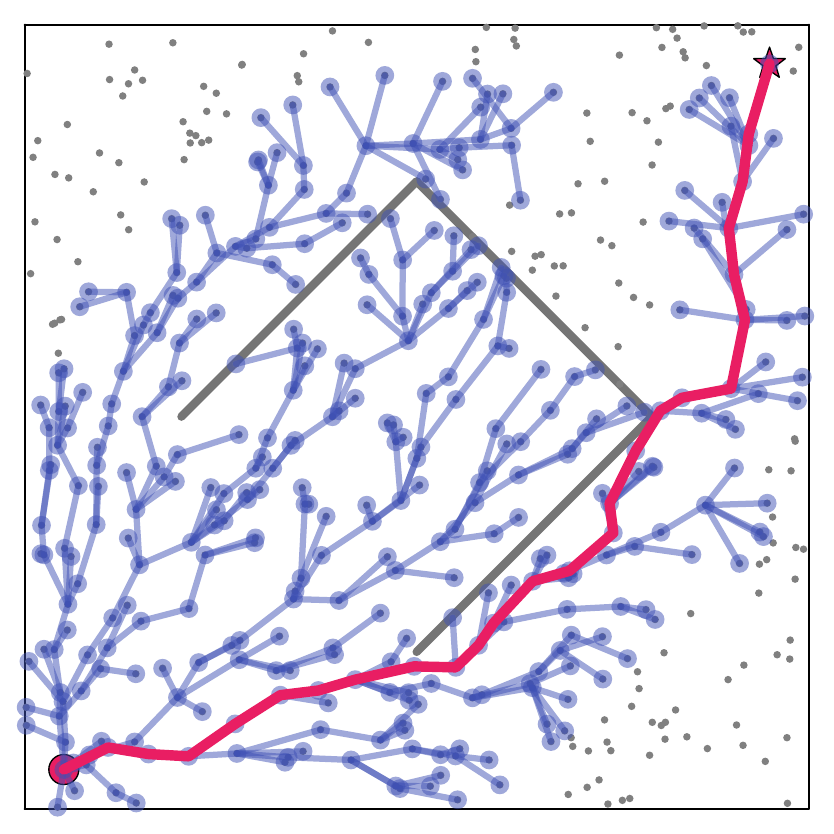}{\astark}{023EFF} &
      \entry{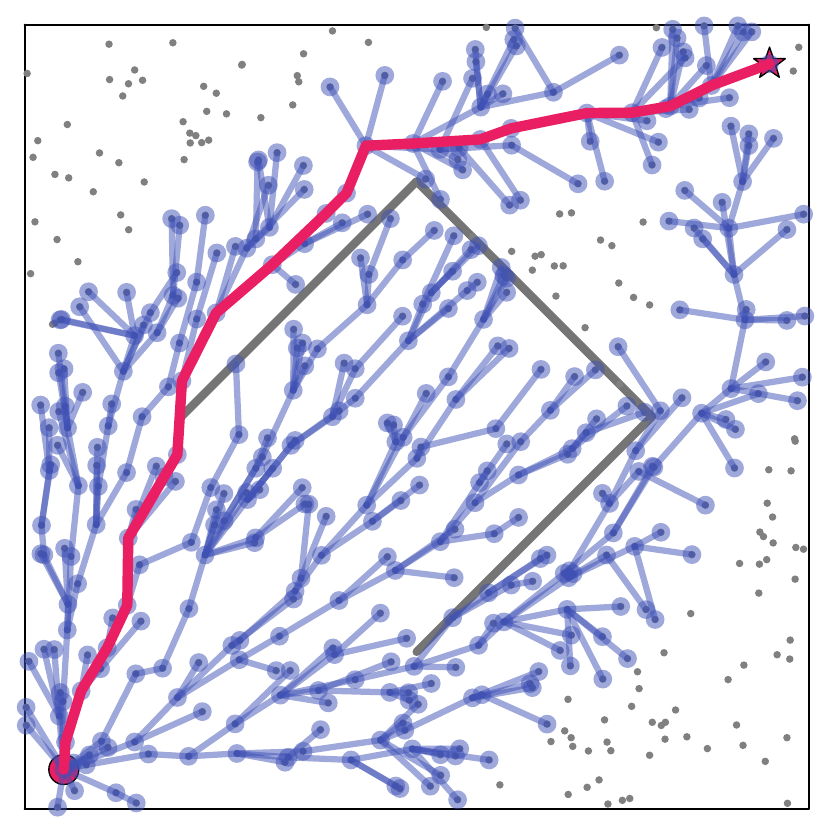}{\astarr}{023EFF} &
      \entry{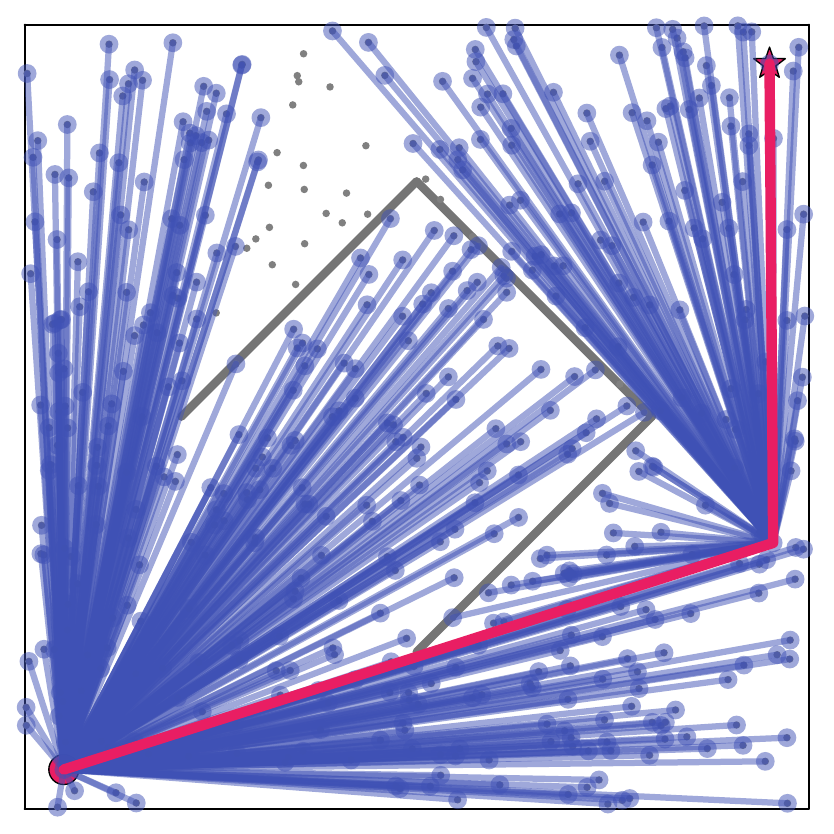}{GBFS}{1AC938} &
      \entry{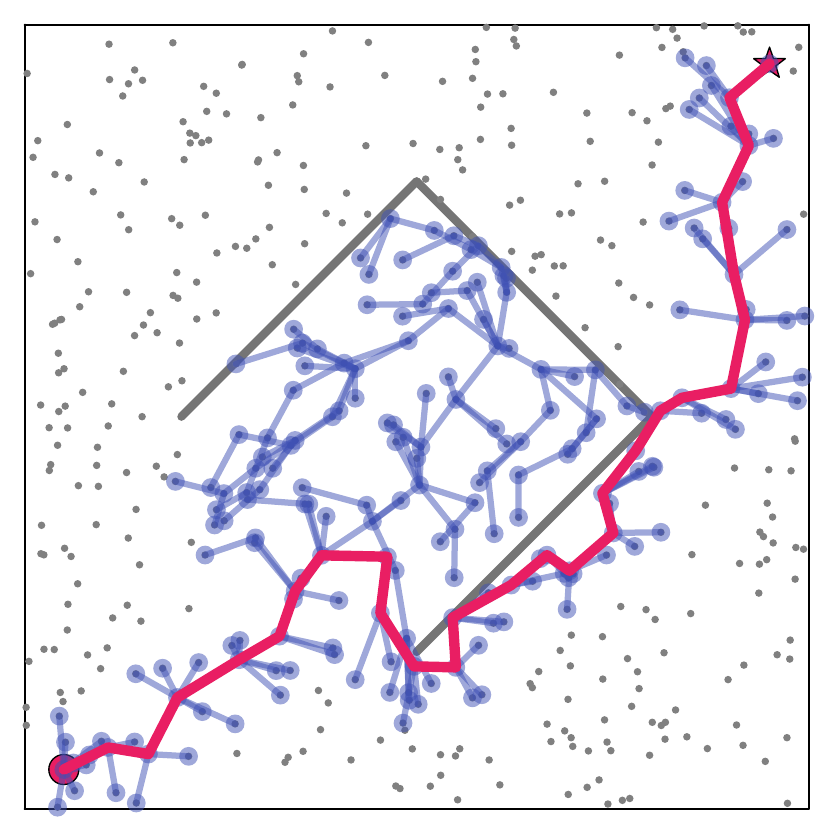}{GBFS-k}{1AC938} &
      \entry{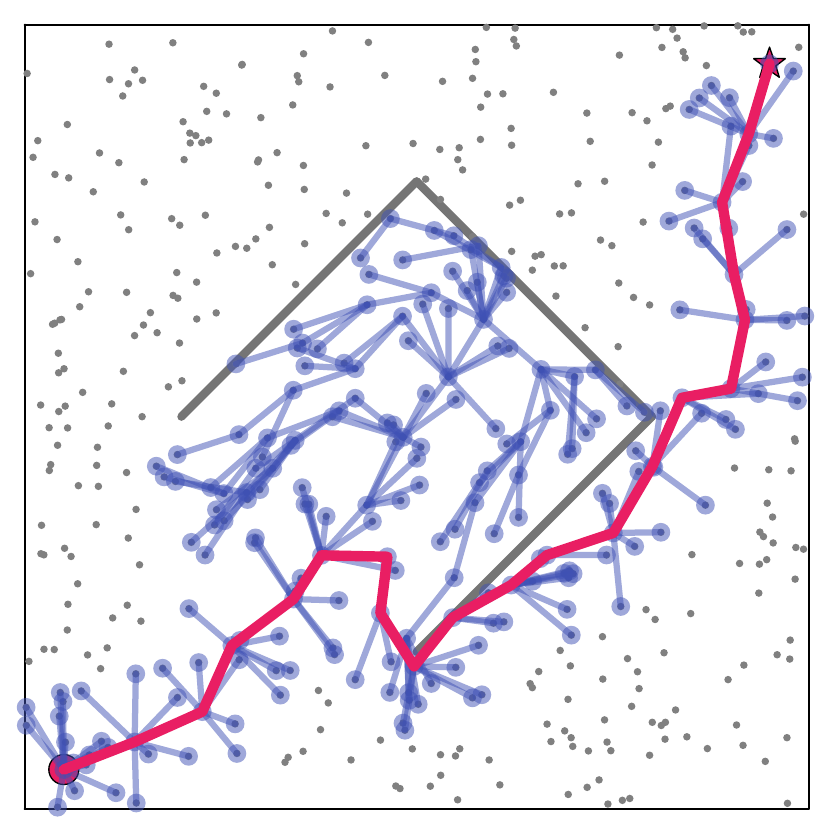}{GBFS-r}{1AC938} &
      \entry{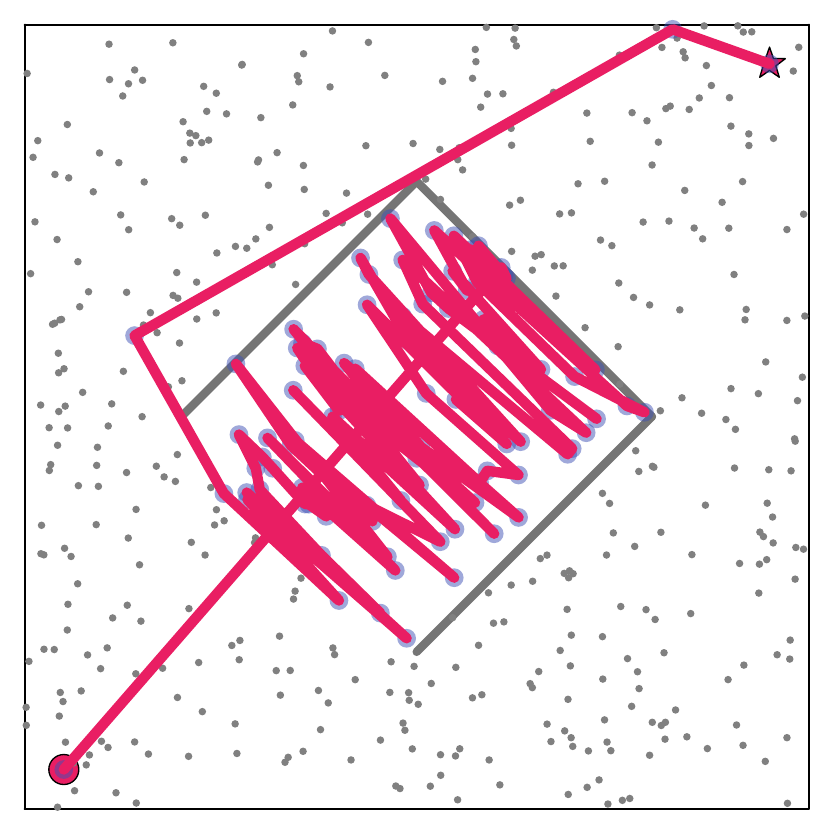}{DFS}{E8000B} &
      \entry{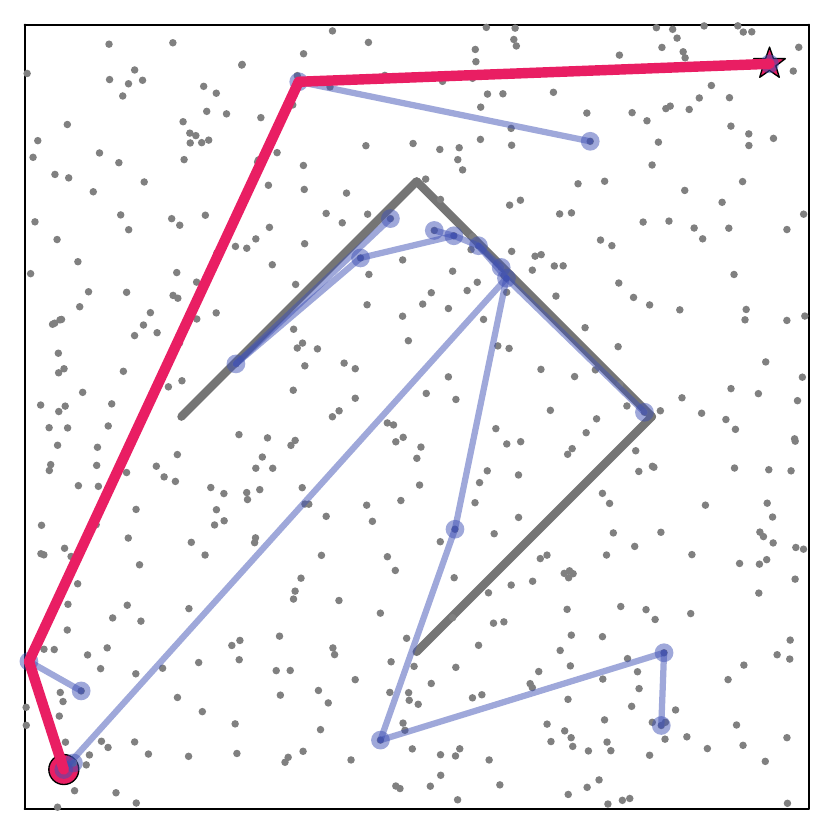}{dRRT}{8B2BE2}
      \medskip\\
      \entry{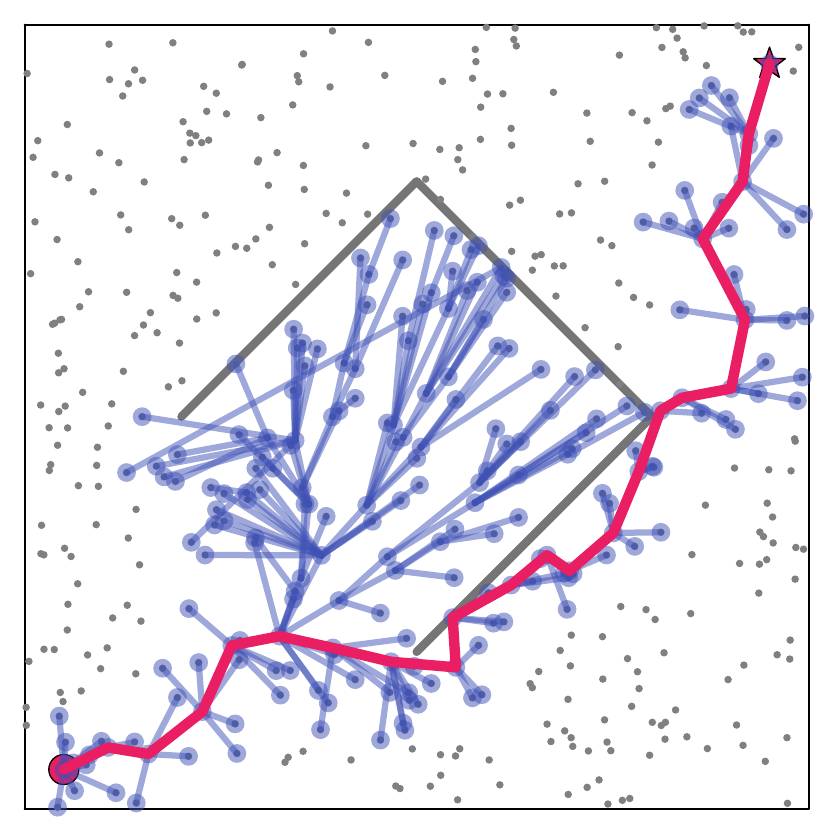}{PE}{FFC400} &
      \entry{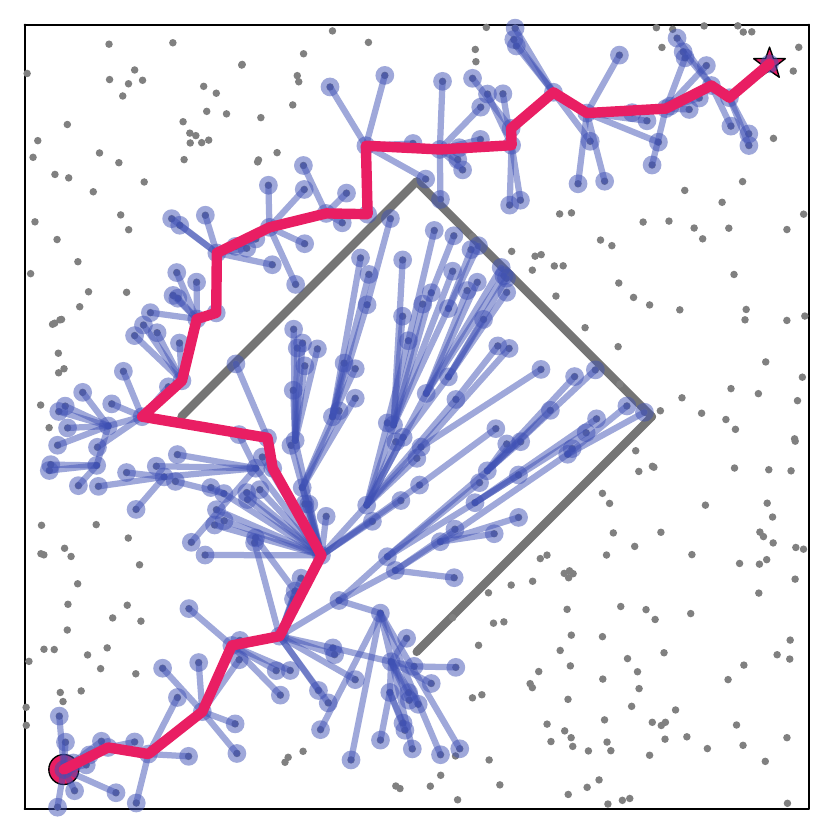}{LaCAS}{FFC400} &
      \entry{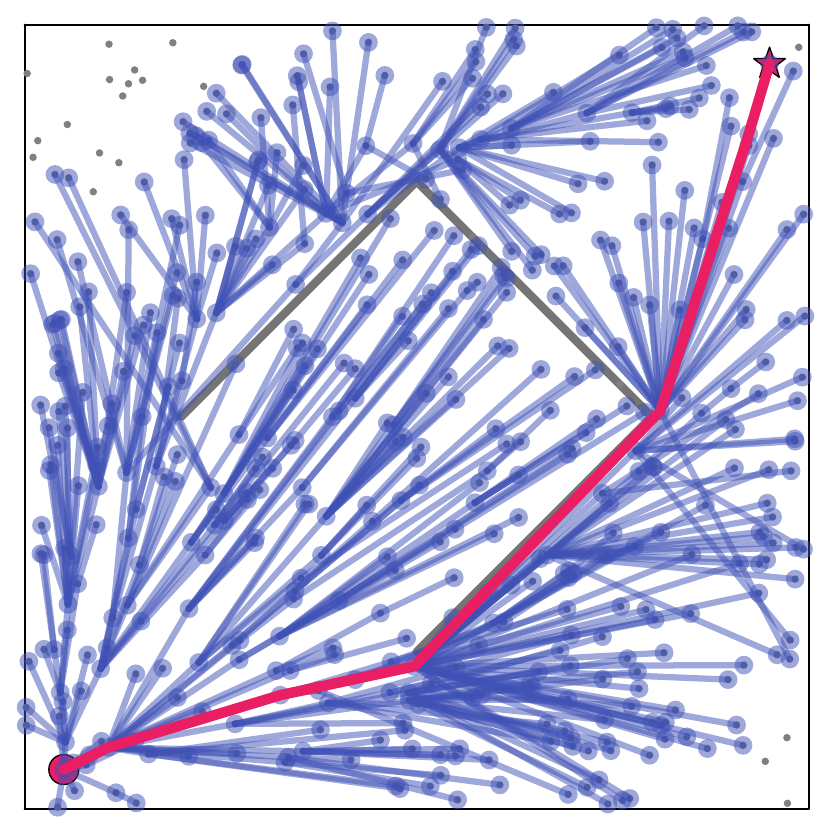}{\lacas}{FFC400} &
      \entry{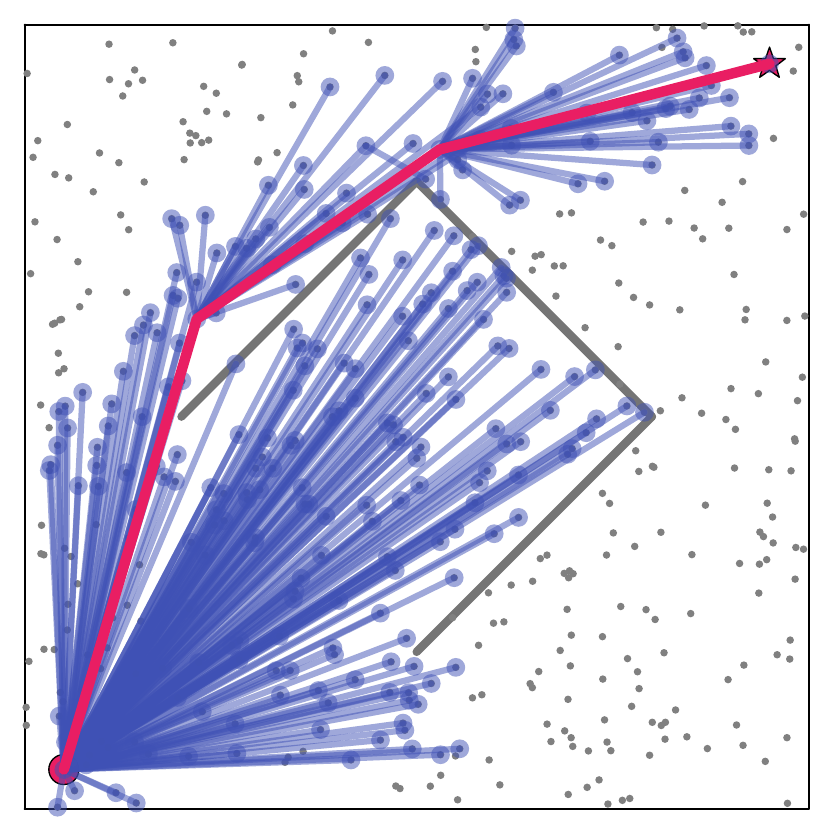}{LaCAT}{FFC400} &
      \entry{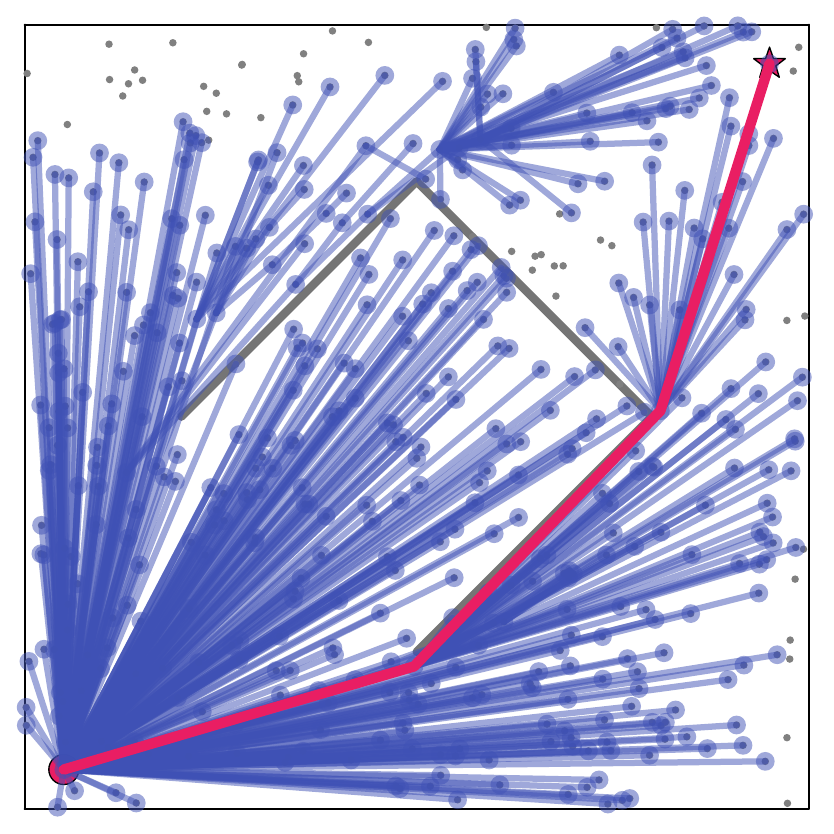}{\lacat}{FFC400} &
      \entry{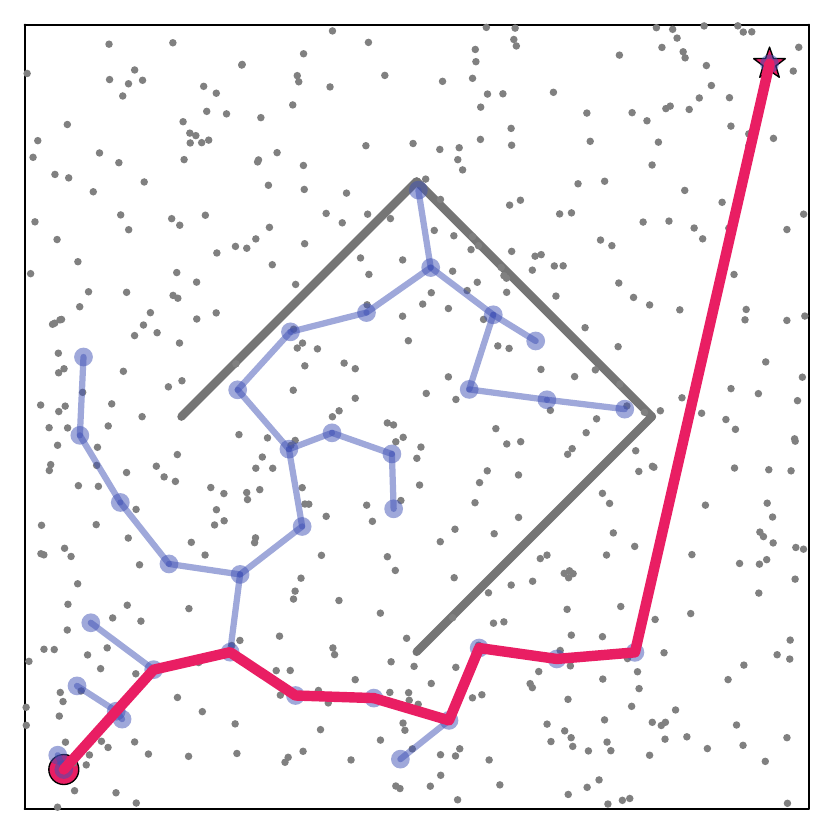}{RRT}{00D7FF} &
      \entry{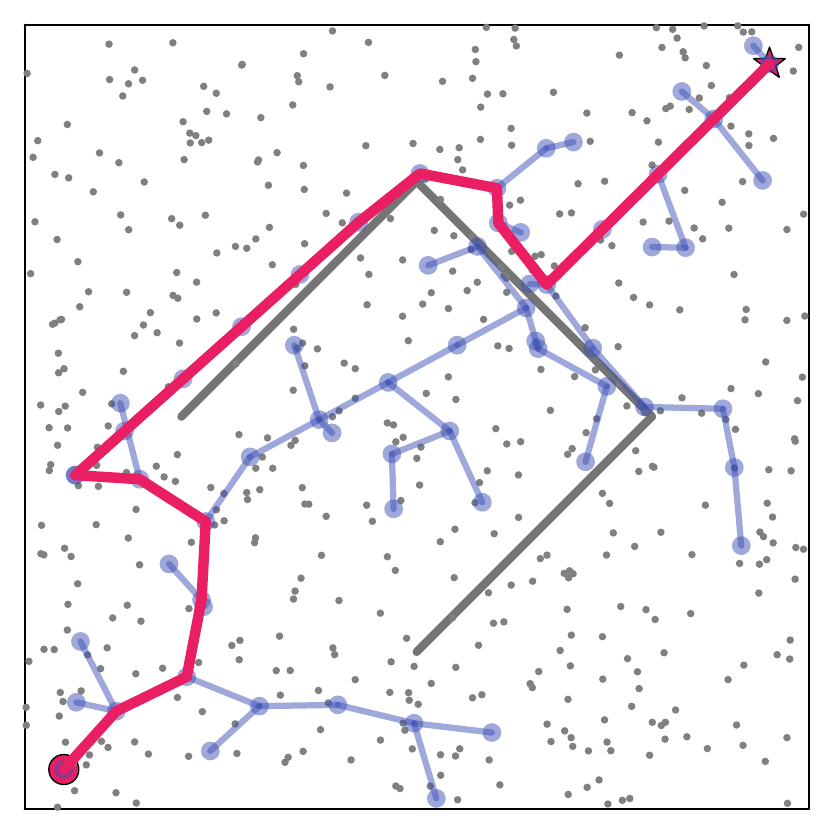}{RRT-C}{00D7FF} &
      \entry{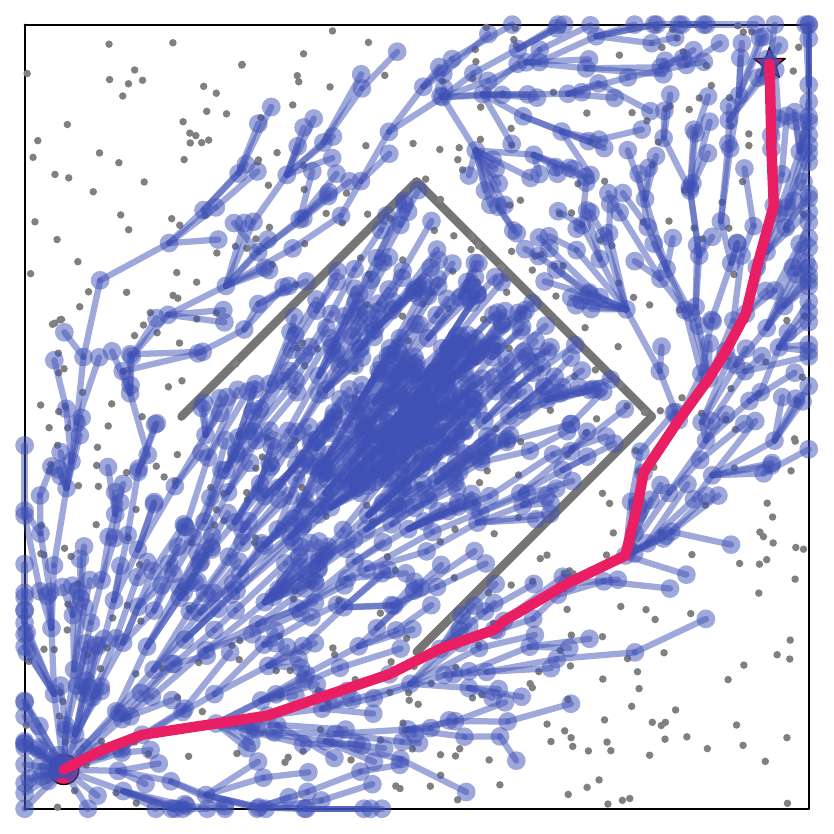}{I\rrtstar}{00D7FF}

    \end{tabular}
    \caption{
    Example run of each method.
    The \scen{trap} scenario is used, with the \SI{1}{\second} timeout and $|V|=100$.
    }
    \label{fig:search-tree}
  \end{figure*}
}

The second type comprises single-query sampling-based motion planning (SBMP) methods, which \emph{neglect} locations defined in \cref{def:problem}.
These methods generate locations on-the-fly by random sampling.
SBMP is difficult to compare fairly with the first class, yet we evaluated this class to offer general insights into pathfinding algorithms.
\begin{itemize}
\item \textbf{RRT}~\cite{lavalle1998rapidly}, a representative SBMP, which is \emph{probabilistically complete}.
  At the implementation level, we performed a goal-connection check at each location addition to accelerate solution finding.
  Additionally, we set the maximum distance for connecting two locations, except for the goal, to $0.1$.
\item RRT-Connect (\textbf{RRT-C})~\cite{kuffner2000rrt}, a well-known extension of RRT, conducting a bi-directional search from both start and goal locations.
\item Informed \rrtstar (\textbf{I\rrtstar})~\cite{gammell2014informed}, another RRT variant, which is \emph{asymptotically optimal} that approaches an optimal solution over time.
  It employs biased sampling for accelerated convergence.
\end{itemize}

In summary, 16 methods including \lacas were evaluated.
\Cref{fig:search-tree} visualizes the search progress of each method.

\subsection{Benchmark Generation}
Eight scenarios with specific start-goal locations were prepared, as shown in \cref{fig:result-main-1,fig:result-main-2}.
For each scenario, 100 instances were generated, varying locations, obstacles, or both.
Note that random generation may result in unsolvable instances.
To focus on the algorithmic nature itself, each scenario was designed with line-shaped obstacles in $\W = [0, 1]^2$, allowing the \connect function to be implemented as a relatively lightweight geometric computation.

\Cref{fig:result-main-1} shows four basic scenarios:
\emph{(i)}~\scen{scatter-1k}, with $1,000$ locations placed by uniformly at random sampling,
\emph{(ii)}~\scen{scatter-10k}, with $10,000$ randomly placed locations,
\emph{(iii)}~\scen{grid-10k}, where locations are on a grid with a $0.01$ resolution (hence $|V|$$\approx$$10,000$),
and \emph{(iv)}~\scen{plus-2k}, with $2,000$ randomly positioned locations.
The first three scenarios include randomly placed line-shaped obstacles, with their length and quantity calibrated for moderate difficulty.
In contrast, \scen{plus-2k} features ``plus'' shaped obstacles.

\Cref{fig:result-main-2} illustrates four carefully designed scenarios to assess each method's characteristics:
\emph{(v)}~\scen{trap}, featuring a basket-shaped obstacle,
\emph{(vi)}~\scen{zigzag}, where the goal-distance heuristic is less effective,
\emph{(vii)}~\scen{gateways}, requiring navigation through five narrow ``gateways'' to reach the goal,
\emph{(viii)}~\scen{split}, with locations split into two distinct zones, separated by multiple obstacles.
Each scenario includes 1,000 locations, randomly distributed.

{
  \setlength{\tabcolsep}{0pt}
  \newcommand{\imgsize}{0.23\linewidth}
  \newcommand{\topimage}[1]{
    \begin{minipage}{\imgsize}
    \centering
    {\scen{~~~#1}}\smallskip\\
    ~~~~
    \includegraphics[trim=11 11 11 11, clip,width=0.7\linewidth]{fig/raw/ins-#1}
    \\
    \includegraphics[width=1\linewidth]{fig/raw/res_#1}
    \end{minipage}
  }
  \newcommand{\rowtitles}{
  \begin{minipage}{0.03\linewidth}
  \begin{tikzpicture}
  \node[]() at (0,1.9) {\rotatebox{90}{solved~{\small (\%)}~$\rightarrow$}};
  \node[]() at (0,-0.6) {\rotatebox{90}{$\leftarrow$~runtime~{\small (\SI{}{\second})}}};
  \node[]() at (0,-3.2) {\rotatebox{90}{$\leftarrow$~cost}};
  \node[]() at (0,-5.7) {\rotatebox{90}{$\leftarrow$~\#connect}};
  \node[color=white]() at (0, 6) {.};
  \node[color=white]() at (0, -7) {.};
  \end{tikzpicture}
  \end{minipage}
  }
  \begin{figure*}[t!]
    \centering
    \small
    \begin{tabular}{ccccccc}
      \rowtitles&
      \topimage{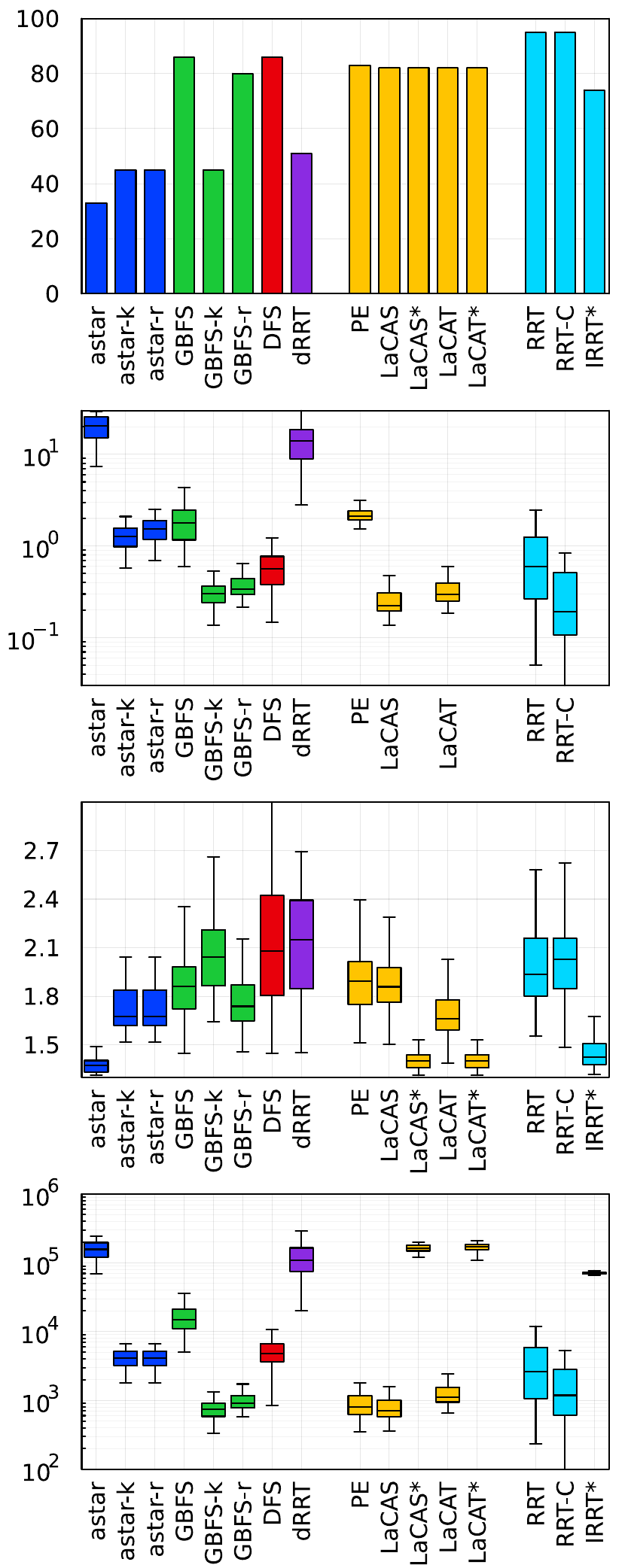} &
      \topimage{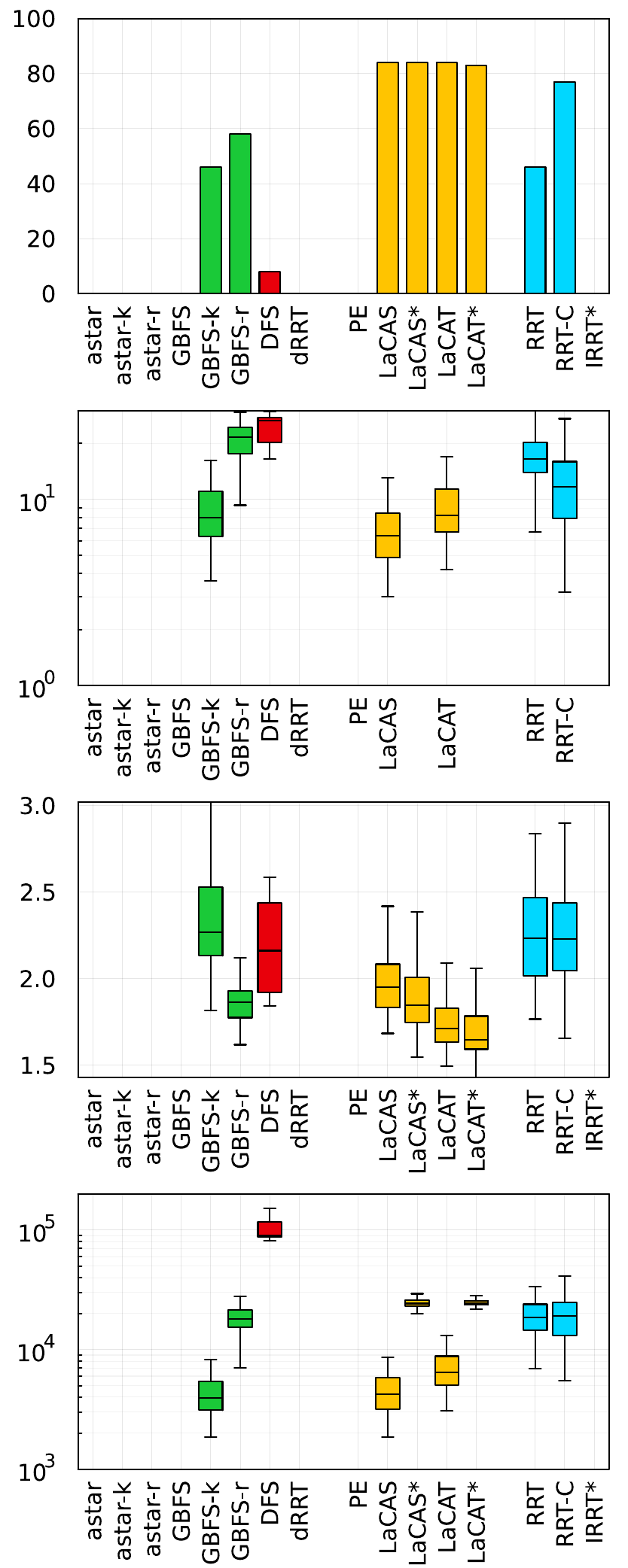} &
      \topimage{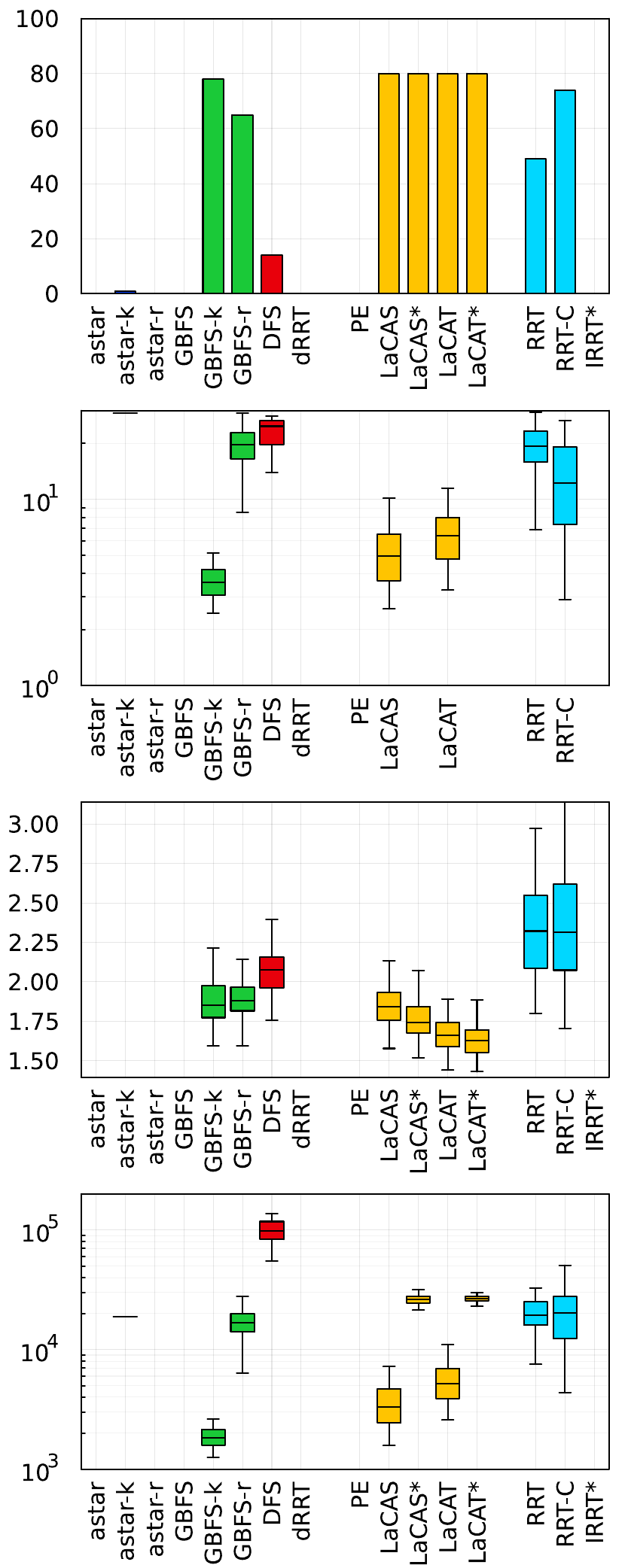} &
      \topimage{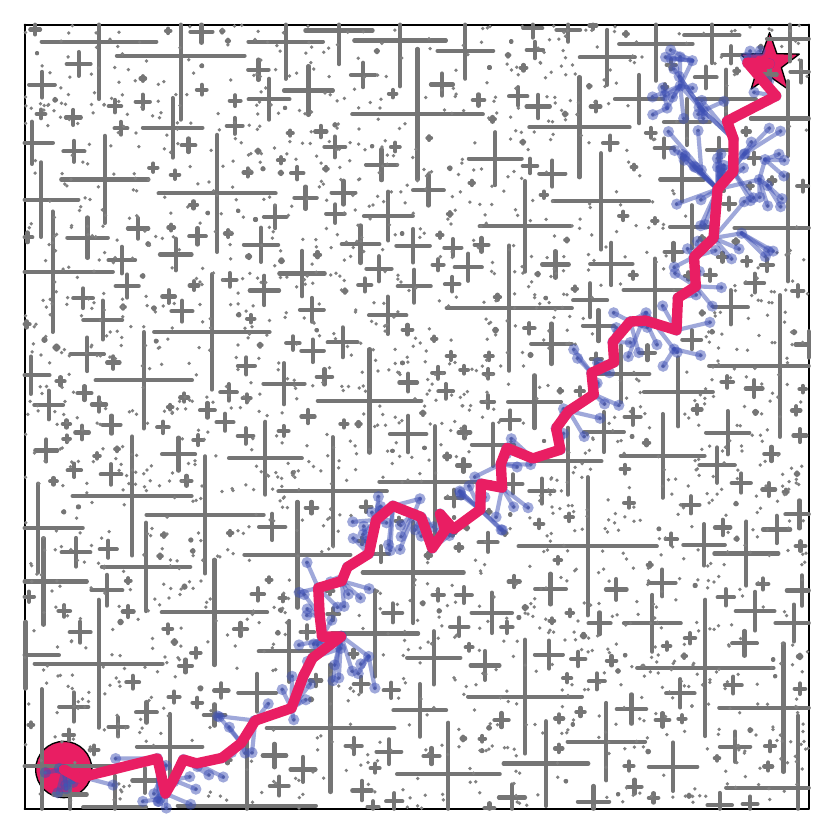} &
    \end{tabular}
    \caption{
    Results of basic scenarios.
    Boxplots show the median values with the first and third quartiles.
    In addition to the runtime results relying on implementations, the figure provides the number of \connect calls (\#\connect) as a proxy for the search effort.
    Runtime scores of \lacas, \lacat, and I\rrtstar are omitted because they were executed until reaching the time limit.
    }
    \label{fig:result-main-1}
  \end{figure*}
  \begin{figure*}[t!]
    \centering
    \small
    \begin{tabular}{ccccccc}
      \rowtitles&
      \topimage{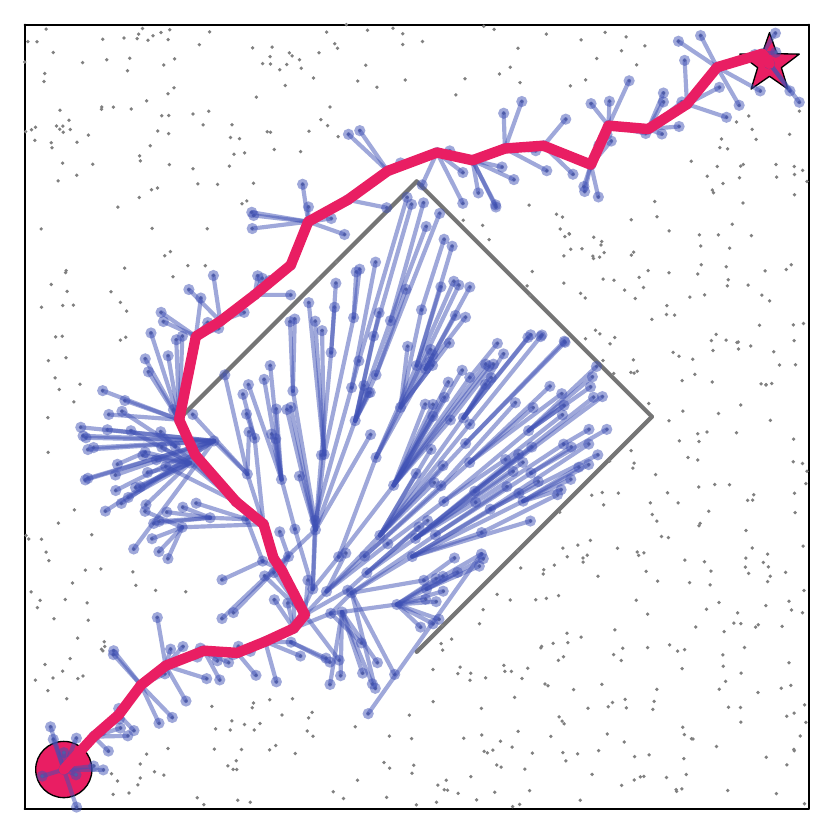} &
      \topimage{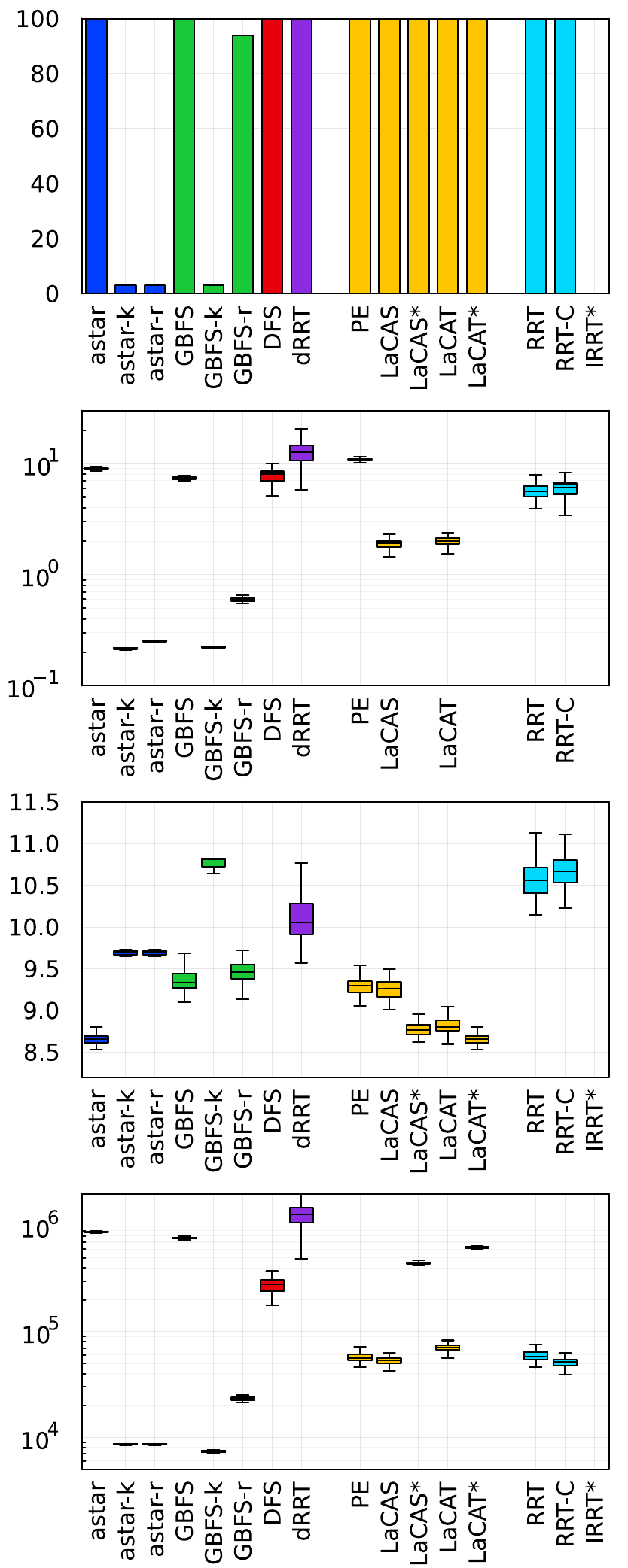} &
      \topimage{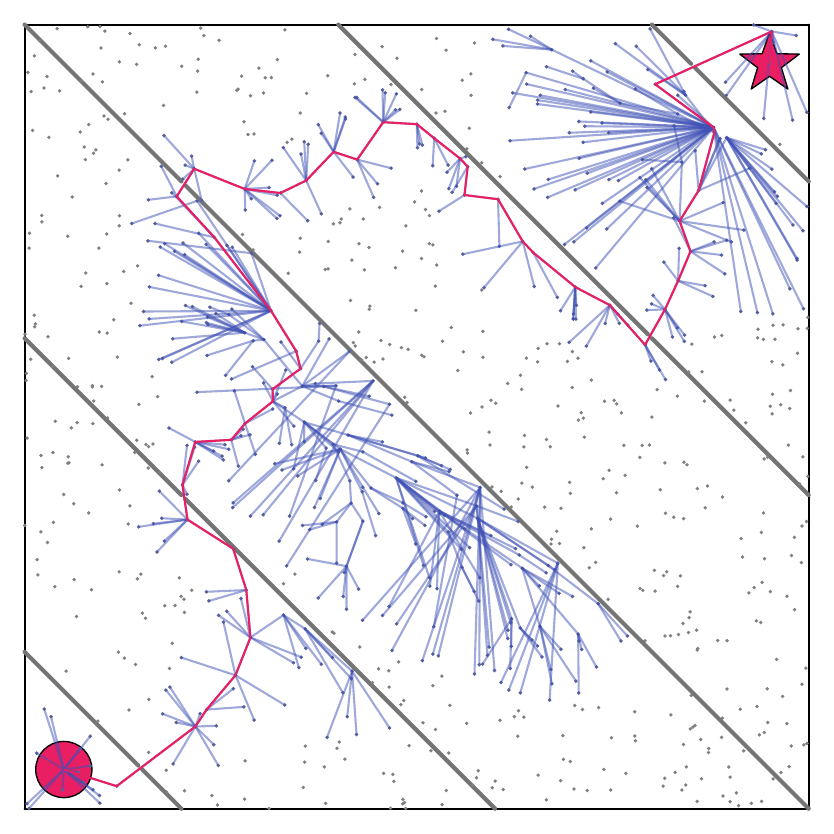} &
      \topimage{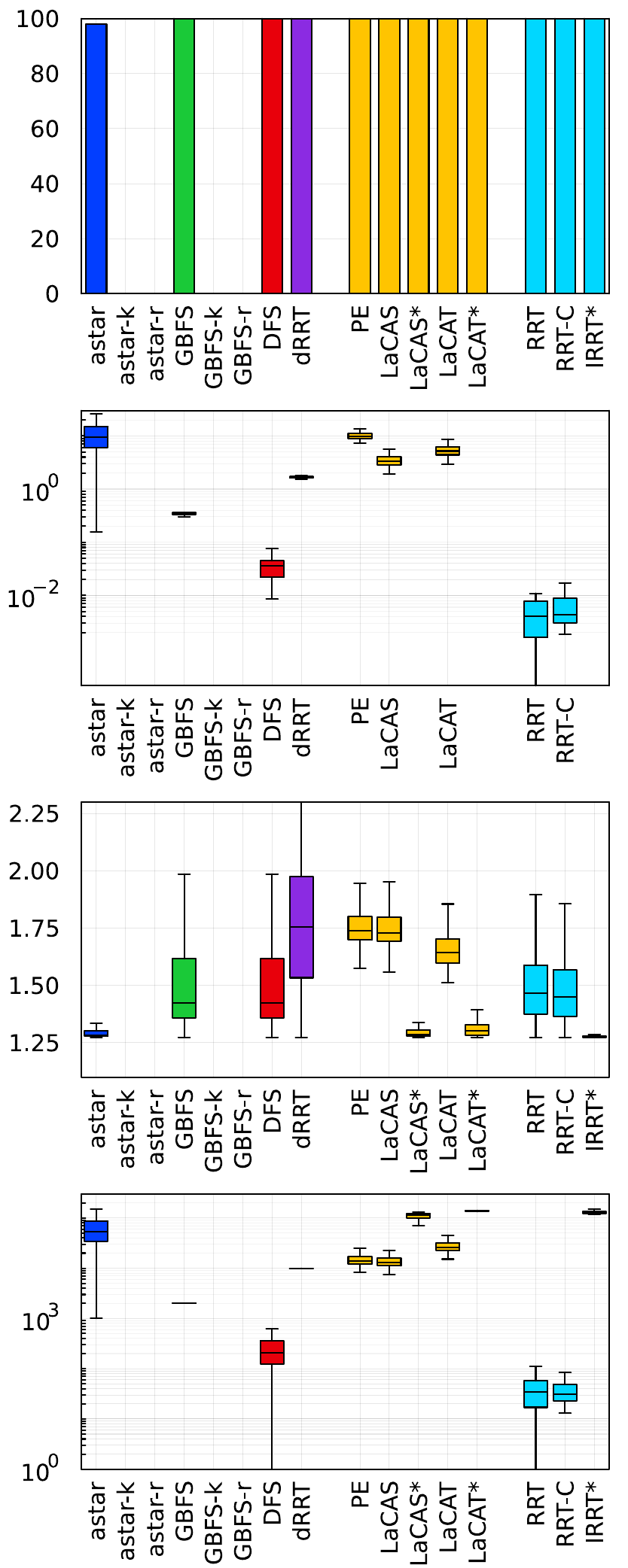}
    \end{tabular}
    \caption{
    Results of specially designed scenarios.
    }
    \label{fig:result-main-2}
  \end{figure*}
}

\subsection{Results and Discussions}
\Cref{fig:result-main-1,fig:result-main-2} show that, overall, \lacas achieved superior performance in pathfinding with the huge branching factor, while maintaining theoretical guarantees of completeness and eventual optimality.
LaCAS quickly identified initial solutions with reasonable costs in most scenarios.
Moreover, \lacas continued enhancing solution quality up to the time limit.
LaCAT$^{(\ast)}$ further elevated this quality, albeit with the extra computational overhead of verifying grandparents.
Those performances were owing to the smaller number of calling \connect.
Specific findings are summarized below.

\paragraph{Huge Branching Factor} makes pathfinding challenging, as evidenced by \astar and dRRT struggling to find solutions.
While GBFS and DFS can resolve scenarios like \scen{scatter-1k} when the heuristic is effective, instances such as \scen{scatter-10k} and \scen{grid-10k} mostly remain unsolved.
This is due to numerous ``\scen{trap}'' situations where the heuristic leads the search astray, being fatal combined with the huge branching factor.
LaCAS circumvents this issue by lazy successor generation and node rolling.

\paragraph{Limiting Successors} as in GBFS-k can aid in pathfinding, yet this approach is not a fundamental solution due to parameter sensitivity.
An extreme case is \scen{split}, where limiting successors proves entirely falter.
Instead, LaCAS dynamically adjusts the number of successors through lazy successor generation, enabling it to find solutions across various scenarios.

\paragraph{Two-level Search Scheme in LaCAS} demonstrates its effectiveness compared to PE.
LaCAS enhances speed by bypassing the evaluation of every potential successor.

\paragraph{SBMP} methods like RRT differ from LaCAS in assuming predefined locations.
On-the-fly location sampling in SBMP can sometimes mitigate unnecessary search space expansion, enabling quick pathfinding; see \scen{trap} or \scen{split}.
However, their performance is highly dependent on how well the sampling strategy is matched to the environment.
A typical mismatch example is seen in \scen{gateways}, where SBMP methods struggle in sampling two suitable points to pass through narrow corridors.
LaCAS avoids this problem by using a preset mass of locations effectively.

\paragraph{Rewiring Overhead} exists, as can be seen by comparing I\rrtstar with RRT.
Analogous to this, LaCAS can achieve speedup by omitting the rewiring process, i.e., gray parts in \cref{algo:lacas}, if the objective is to just find feasible solutions.

\paragraph{Challenging Situations for LaCAS} are illustrated in \cref{fig:result-main-2}.
For instance, \scen{zigzag} demonstrates performance decline when the heuristic is ineffective, whereas \scen{split} exemplifies the drawbacks of lazy successor generation.
Nonetheless, LaCAS consistently identifies solutions with acceptable costs within the allocated time budget.

\subsection{Effect of Batch Size}
\Cref{fig:batch-size} illustrates the impact of batch size on LaCAS's performance, indicating that there are sweet spots in batch size.
For very small batch sizes, search progression can be slow, as successors closer to the goal might not appear in the initial successor generation for each node.
Conversely, large batch sizes diminish the benefits of lazy successor generation, which is pivotal for rapid search.

{
  \begin{figure}[t!]
    \centering
    \includegraphics[width=0.6\linewidth]{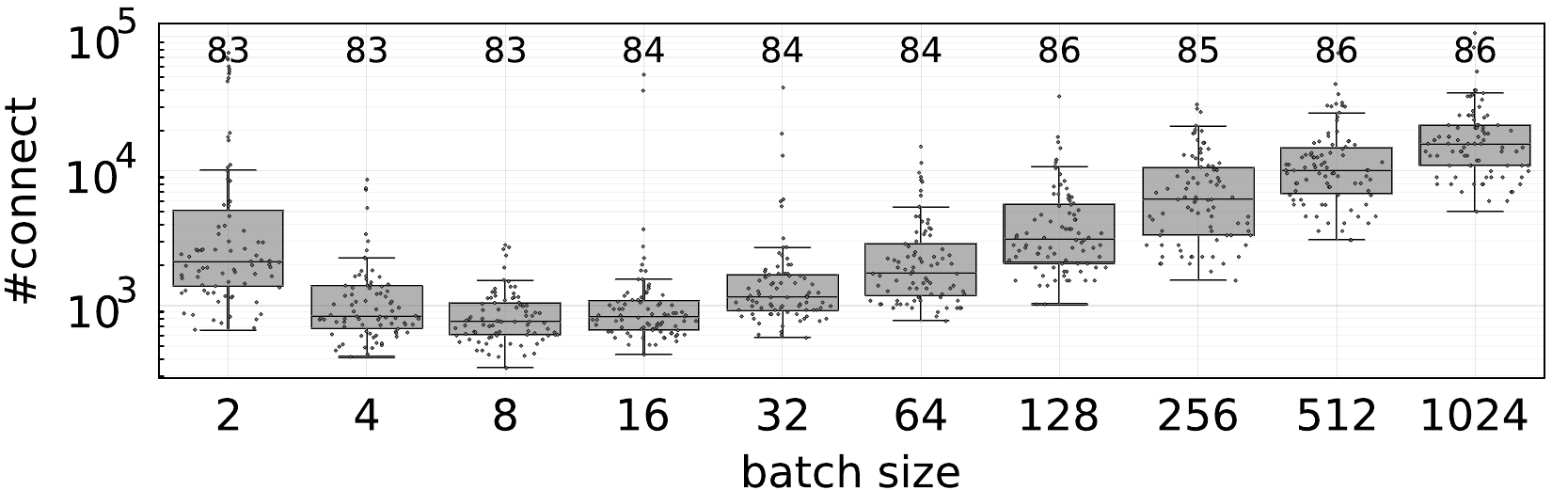}
    \caption{
      Effect of batch size on search effort of LaCAS with \scen{scatter-1k}.
      The number of solved instances within \SI{30}{\second} over 100 instances is displayed at the top of the figure.
    }
    \label{fig:batch-size}
  \end{figure}
}

\section{Conclusion}
\label{sec:discussions}
This paper studied pathfinding on a graph where every vertex is potentially connected to every other vertex, and the connectivity is implicitly defined by an oracle.
The challenge arises from the huge branching factor, rendering standard search algorithms computationally expensive.
While limiting successor generation mitigates this issue, it compromises the complete and optimal search structure, and its efficacy heavily relies on design choices.
Our study addresses this dilemma through \lacas.
At its core, \lacas employs lazy successor generation via the two-level search.
This trick enables \lacas to maintain both empirical effectiveness and theoretical assurances of completeness and eventual optimality.

The \lacas concept originates from \lacam, one of the cutting-edge multi-agent pathfinding implementations~\cite{okumura2024engineering}, which stands out as it achieves stable planning for thousands of agents.
It shares the fundamental principle of \emph{combinatorial search with generators}, a scheme that gradually generates small portions of the entire successors.
While generators need specific designs for each problem domain, the encouraging outcomes in both single- and multi-agent pathfinding suggest this approach is an effective strategy for planning problems with huge branching factors.

\section*{Acknowledgments}
This work was partly supported by JST ACT-X Grant Number JPMJAX22A1.

\bibliographystyle{named}
\bibliography{ref-macro,ref}

\end{document}